\documentclass[12pt,a4article]{article}

\usepackage{amsthm}

\usepackage{amsmath}
\usepackage{amsfonts}
\usepackage{amssymb}
\usepackage{graphicx}
\usepackage{caption}
\usepackage{subcaption}
\usepackage{MnSymbol,wasysym}
\usepackage{cancel}
\usepackage{float}
\usepackage{splitbib}
\usepackage{color}
\usepackage{yfonts}
\usepackage{enumitem}
\theoremstyle{definition}
\newtheorem{defn}{Definition}[section]
\newtheorem{thm}{Theorem}[section]
\newtheorem{lem}[thm]{Lemma}

\theoremstyle{remark}
\newtheorem{remark}{Remark}[section]

\author{Ilana Segall and Alfred Bruckstein}  \date{Center for Intelligent Systems\\MultiAgent Robotic Systems (MARS) Lab\\Computer Science Department\\
Technion, Haifa 32000, Israel\\
\today}

\title{ Guidance of Agents in Cyclic Pursuit\thanks{This research was partly supported by Technion Autonomous Systems Program (TASP)} }

 \newtheorem{theorem}{\textit{Theorem}}
 \newtheorem{lemma}{\textit{Lemma}}
\newtheorem{corollary}{\textit{Corollary}}

\newtheorem{definition}{\textit{Definition}}

\begin{document}
\maketitle
\newpage

\tableofcontents

\newpage

\begin{abstract}
This report studies the emergent behavior of systems of agents performing  cyclic pursuit controlled by an external  broadcast  signal detected by a random set of the agents. Two types of cyclic pursuit are analyzed: 1)linear cyclic pursuit, where each agent senses the relative position of its target or leading agent 2)non-linear cyclic pursuit, where the agents can sense only bearing to their leading agent and colliding agents merge and continue on the  path of the pursued agent (a so-called "bugs" model). Cyclic pursuit is, in both cases, a gathering algorithm, which has been previously analyzed. The novelty of our work is the derivation of emergent behaviours, in both linear and non-linear cyclic pursuit, in the \emph{ presence of an exogenous broadcast control detected by a random subset} of agents. We show that the emergent behavior of the swarm depends on the type of cyclic pursuit. In the linear case, the agents asymptotically align in the desired direction and move with a common speed which is a proportional to the ratio of the number of agents detecting the broadcast control to the total number of agents in the swarm, for any magnitude of input (velocity) signal. In the non-linear case, the agents gather and move with a shared velocity, which equals the input velocity signal, independently of the number of agents detecting the broadcast signal.

\end{abstract}

\newpage

\section*{Symbols and Abbreviations}
$n$ - Number of agents\\
$p_i$ - position of agent $i$, $p_i \in \mathbb{R}^2$\\
$x_i$ - $x$ coordinate of $p_i$\\
$y_i$ - $y$ coordinate of $p_i$\\
$P$ - vector of stacked positions of all agents, $P=(p_1,...., p_n)^T$\\
$X$ - vector of $x$ coordinates, $X=(x_1, \dots, x_n)^T$\\
$Y$ - vector of $y$ coordinates, $Y=(y_1, \dots, y_n)^T$\\
$u_i$ - local gathering control applied by agent $i$, $u_i \in \mathbb{R}^2$\\
$U_c$ - external broadcast control, $U_c \in \mathbb{R}^2$\\
$b_i$ - flag indicating whether agent $i$ detected the broadcast control, $1/0$\\
$B$ - vector indicator of agents detecting the broadcast control, $B(i)=b_i$\\
$N^l$ - the set of agents detecting the broadcast control\\
$n^l$ - the number of agents detecting the broadcast control, $n^l = |N^l|$\\
$\|v\|$ - Euclidean norm of vector $v$\\
$|s|$ - absolute value of scalar $s$\\
$d_i$ - distance of agent $i$ from $i+1$, $d_i=\|p_i-p_{i+1}\|$\\
$\overline{(.)}$, conjugate of $(.)$, scalar or vector\\
$v^*$ - transpose conjugate of vector $v$, $v^*=(\overline{v})^T$\\
$\mathbf{1}_n$ - vector of ones, size $n$\\
$\mathbf{0}_n$ - vector of zeroes, size $n$\\

\newpage
\section{Introduction}
In this work we consider the behaviour of $n$ agents performing cyclic pursuit, when an exogenous velocity control is broadcast by a controller and detected by a random set of agents. The cyclic pursuit problem is formulated as $n$ agents chasing each other.  The agents are ordered
from 1 to $n$, and agent $i$  acquires information about its leading agent (prey) $i + 1$. The agent indices are (modulo $n$) throughout this paper. The agents start
from  arbitrary positions on a plane.

 In graph theoretic terms,  cyclic pursuit can be represented by a directed cycle graph, whose nodes are
the agents and the directed edges depict the information flow, as shown in Fig. \ref{fig:Cyclic flow}.
\begin{figure}[ht]
\centering
\includegraphics[scale=0.4]{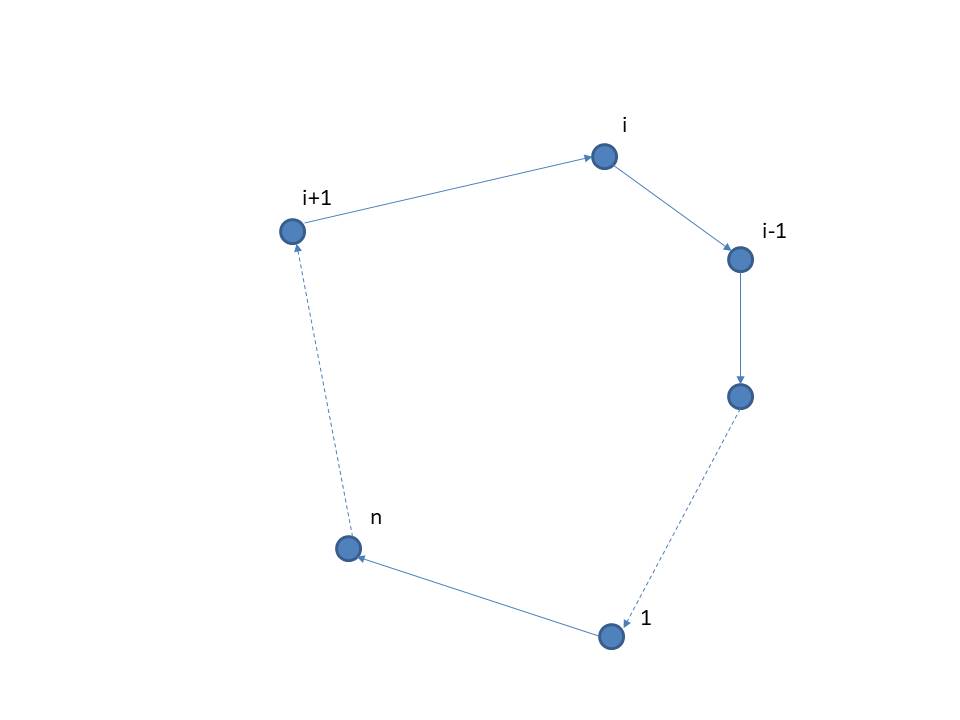}
\caption{Information flow in cyclic pursuit}
\label{fig:Cyclic flow}
\end{figure}

In this work we consider two types of information possibly acquired by the agents
\begin{enumerate}
\item Relative position to the chased "target" agent 
\item Bearing only information, i.e. direction to the target
\end{enumerate}
Let $p_i(t)$ be the position of agent $i$ at time $t$;  $p_i(t) \in \mathbb{R}^2$.  We assume the agents to be identical, memory-less particles, modeled as single integrators.
\begin{enumerate}
\item In case of relative position information, the autonomous kinematics of agent $i$ can be expressed as 
\begin{equation}\label{dot_p_i-h}
\dot{p}_i(t) = k (p_{i+1}(t) - p_i(t); \text{  }k>0  \quad i=(1,...n)
\end{equation}
In the sequel we assume, without loss of generality, $k=1$.  It's easy to see that in this case (\ref{dot_p_i-h}) can be decoupled into 
\begin{eqnarray}
\dot{x}_i(t) & = & x_{i+1}(t)-x_i(t)\label{dot-x_i}\\
\dot{y}_i(t) & = & y_{i+1}(t)-y_i(t)\label{dot-y_i}
\end{eqnarray}
Thus, we can consider only the $x$ coordinates and obtain results for the $y$ coordinates by similarity. 
Let $X(t)=(x_1(t), \dots, x_n(t))^T$. Then, from (\ref{dot-x_i}), we have
\begin{equation}\label{dot-X}
\dot{X}(t) = M X(0)
\end{equation}
Eq. (\ref{dot-X}) is a linear system, where $M$ is the circulant matrix (\ref{eq-M-1}). 
   \begin{equation}\label{eq-M-1}
   M  = circ[ -1, 1, 0, 0, \hdots, 0]=
          \begin{bmatrix}
           -1 & 1 & 0 & 0 & \hdots & 0 \\
            0 & -1 & 1 & 0 & \hdots & 0 \\
	   &&\vdots \\
            1 & 0  && \hdots& 0  & -1 \\
          \end{bmatrix}
  \end{equation}
This system shall be referred to in the sequel as \emph{linear cyclic pursuit}.
 
\item Bearing only information flow generates a non-linear cyclic pursuit. Let $d_i(t)$ be the distance between agents $i$ and $i+1$ at time $t$:
\begin{equation}\label{eq-d_i-0}
d_i(t)=\|p_{i+1}(t) - p_i(t)\|
\end{equation}
where $\|.\|$ represents the Euclidean norm. 
Then, if $d_i(t) \neq$ 0 the law of bearing-only autonomous motion can be written as
\begin{equation*}
\dot{p}_i(t) = \frac{p_{i+1}(t)-p_i(t)}{d_i(t)}
\end{equation*}
where we assumed the speed of all agents to be 1. Moreover, we assume that if $d_i(\hat{t}) = 0$ at some time $\hat{t}$, then $p_i(t)  =  p_{i+1}(t) \text{  for all  } t \geq \hat{t}$, i.e.   
when agents $i$ and $i+1$ collide they merge and continue as one agent in the direction of $i+1$.   This model is known in the literature as the "bugs" model, see e.g. \cite{Richardson}, \cite{KN} and references therein .
\end{enumerate}

The novelty of our work is the derivation of emergent behaviours, in both linear and non-linear cyclic pursuit, in the \emph{ presence of an exogenous broadcast velocity signal detected by a random subset} of agents.
The impact of the external velocity signal on the movement of agents in the linear case is shown in section \ref{LinCyclic} and for the non-linear case in section \ref{NonLinCyclic}.  

\subsection{Literature survey and our contribution}

\subsubsection{Linear cyclic pursuit }
Autonomous linear cyclic pursuit,  belongs to the larger class of networks with directed, fixed topology graphs, denoted by $G$, where each node applies protocol (\ref{gen-u_i})
\begin{eqnarray}
\dot{x}_i(t) &=& u_i(t)\label{gen-u_i-0}\\
u_i(t)&=&\sum_{j \in N_i} a_{ij}(x_j(t)-x_i(t)); \quad i \in \{1,\dots, n\}\label{gen-u_i}
\end{eqnarray}
where $N_i$ is the neighborhood of $i$, defined as $N_i = \{j \in  \{1,\dots, n\} \text{:  }a_{ij}>0 \}$. In our case $N_i=\{i+1 \text{  mod}(n)\}$, $a_{ij}=1 \forall i$ and $G$ is as depicted by Fig. \ref{fig:Cyclic flow}.  
Olfati-Saber, Fax and Murray  show in \cite{OS-M2003}, \cite{OS-M-TR2003} that a network of integrators with
directed information flow, $G$,  that is strongly
connected, using Protocol (\ref{gen-u_i}), yields the following results:
\begin{itemize}
\item It globally asymptotically solves an agreement problem, i.e. $x_i(t \rightarrow \infty) = x_j(t \rightarrow \infty) = \alpha; \quad \forall i,j $, [see Proposition 2 in \cite{OS-M2003})
\item A sufficient condition
for $\alpha = Avg(x(0))$, i.e. the agreement to be the average agreement, is $\sum_{i=1}^n u_i =0$.\\ 
Note that if $G$ is undirected and symmetric, i.e. $a_{ij}= a_{ji}$, then the condition $\sum_{i=1}^n u_i =0$  automatically holds and $Avg(x(t))$ is an invariant quantity, see \cite{Saber}
\item $G$ globally asymptotically solves the
average-consensus problem using Protocol (\ref{gen-u_i}) if and only if $G$ is balanced.
 
\end{itemize}
 
Recalling that
\begin{itemize}
 \item A digraph is called strongly connected if for every pair of vertices there is a \emph{directed path} between them.
   \item A node is called balanced if  the total weight of edges entering the node and leaving the same node are equal
    \item If all nodes in the digraph are balanced then the digraph is called balanced
\end{itemize}
 
We observe that the circular flow graph, depicted in Fig.  \ref{fig:Cyclic flow}, is strongly connected, balanced and,  using $u_i(t) =x_{i+1}(t)-x_i(t)$, satisfies $\sum_{i=1}^n u_i =0$, the system described by \ref{dot-x_i}, \ref{dot-y_i} solves the average consensus problem.

Addressing specifically the problem of linear cyclic pursuit, other researchers derived similar results.
 Bruckstein et al., in \cite{Bruck}, see Section "Linear Insects", showed that for every initial condition, the agents exponentially converge to a single point, computable from the initial conditions of the agents.
 Marshal, in \cite{MarshallPhD} Section 2.2.6, offers an alternate proof for the same.  

  If the agents apply heterogeneous gains to the   $u_i$, i.e. $u_i(t)=k_i(p_{i+1}(t) - p_i(t)); \text{  } k_i > 0 \text{  } \forall i $,  then the point of convergence of the agents can be controlled  by selecting these gains, as shown in \cite{S-G-TR}. Moreover, if convergence to a point is achievable, then other formations are achievable by a simple modification, where
each agent pursues a displaced version of the next agent, as discussed in  \cite{LBF2003}.

All of the above studies consider zero input or autonomous systems in cyclic pursuit. Our main contribution is in the addition of an external broadcast control, detected by a \emph{random set} of agents,  and the derivation of the asymptotic behaviour of the system in this case. 
 Ren, Beard, and McLain, in \cite{RBMcL}, consider the problem of dynamic consensus, which at first glance is similar to ours but is only a simple, special case, of our paradigm and results. Applying their general graph case to the cyclic pursuit case, the update law they apply is
 \begin{equation*}
 \dot{x}_i (t)=(x_{i+1} (t)-x_i (t))+U(t)
 \end{equation*}
i.e. a \emph{common input} $U(t)$ is applied at time $t$ to \emph{all} agents. They show (Theorem 3) that in this case  $\|x(t)-\zeta(t)\| \xrightarrow {t \to \infty} 0$, where $\zeta(t)$ is the integral of $U(t)$ starting at the equilibrium point of the autonomous, zero input system. Moore and Lucarelli,  \cite{M-L} consider the case where a separate
input enters each agent. However, they limit their analysis specifically
the case where  an input enters only one node (or agent), say $k$. Moreover, the input is not a general velocity control, as in our paradigm, but attraction of agent $k$ to a goal position, say $x^{sp}$, i.e. 
\begin{equation}\label{u_i-M-L}
  U_i(t)= \begin{cases}
    (x^{sp}-x_i(t)) \quad \text{if  } i=k\\
    0 \quad \text{otherwise}\\
    \end{cases}
\end{equation}
Dimarogonas, Gustavi et al. in  \cite{DGEH}, \cite{GDEH09} also consider the 
global mission of converging to a known destination point, but allow for multiple leaders. 
Leaders are predesignated agents holding the information of the goal destination, thus the external input in this case becomes 
\begin{equation}\label{u_i-D-G}
  U_i(t)= \begin{cases}
    (x^{sp}-x_i(t)) \quad \text{if  } i \in K\\
    0 \quad \text{otherwise}\\
    \end{cases}
\end{equation}
where $K$ is the predesignated set of leaders. 
In \cite{Ren2007}, W. Ren extends the problem of reaching a goal position to that of consensus to a time varying reference state, and shows
necessary and sufficient conditions under which consensus is reached on the time-varying reference state. This is the problem of tracking a time dependent state and not of steering by an external velocity signal received by a random set of agents, as in our case.

In other works, the model is even further from our paradigm. Some assume that the (predesignated) leaders have a fixed state value and do not abide by the agreement protocol. For example,  Jadbabaie et al. in  \cite{JLM} consider Vicsek's discrete model \cite{VC95}, and introduce a leader that moves  with a fixed heading.  
Yet others add \textbf{special agents} to the swarm with the purpose of controlling the collective behavior. In \cite{HLG2006},  \cite{HW2013} these special agents are referred to as  "shills" \footnote{Shill is a decoy who acts as an enthusiastic, internally driven agent, that looks like an ordinary agent,  having the goal to stimulate the participation of others}. The basic local rules of motion of the existing agents in the system are not changed, however the shill does not obey the same local rules  but has a local control of its own, depending on the states of the ordinary agents and a secret goal function.

We recall that in our paradigm the position of the agents is not known to themselves and the leaders are  are not special agents but regular agents, randomly selected from the swarm, obeying the same gathering rule of motion as the remaining agents.  The external input is a velocity signal aimed at steering the swarm in a desired direction and not a goal position for the swarm.  Thus our problem, as well as solution and results, is different from the above discussed cases, covered in the literature.

\subsubsection{Cyclic pursuit with non-linear local control} 
Our paradigm for non-linear cyclic pursuit discussed herein, comprising  sensing direction to prey, chasing along the line of sight with unit speed and capture (merge) upon collision, is commonly referred to in the existing literature as the ‘bugs’ problem, also known as "ants", see e.g.  \cite{Bruck}, \cite{Bruck1993}.  In \cite{Bruck1993} the convergence of $n$ ants in cyclic pursuit to an encounter point is proved. Richardson shows in \cite{Richardson} that the encounter occurs in finite time. In \cite{Bruck}, Bruckstein et. al extend the model
allowing  each bug (ant) to move at different, time dependent speed, $v_i(t)$. They show that integrability of
the speed plays a central role in the emerging behaviour of the swarm. Speed $v_i(t)$ is integrable iff the cumulative distance travelled by ant $i$ at time $t$, $\displaystyle V_i(t) = \int_0^t v_i(s) ds$ holds $V_i(\infty) < \infty $. Constant speeds are not integrable, hence, according to Theorem 1.ii in \cite{Bruck}, the time  of the last ant collision, i.e. termination time, is finite.

The question of simultaneous mutual capture, i.e. the existence of a time $t_c$ such that the distances $d_i(t_c)  \triangleq \|p_{i+1}(t_c) - p_i(t_c)\| =0$ and $d_i(t) > 0 \quad \forall t \in [t_0, t_c)$ for all $i$,  was also investigated.
In 1971 Klamkin and Newman, \cite{KN}, showed that if $n = 3$, the 3 bugs  travel at the same speed and the initial positions
of the bugs are not collinear, then the meeting of the three bugs must be mutual, i.e. all bugs capture their prey simultaneously.
Klamkin and Newman \emph{speculate} that this result generalizes to more bugs. In \cite{BG1979}, Behroozi and Gagnon  prove that it
does indeed generalize to $n = 4$ if the bugs’ initial positions form a convex polygon. Only non-convex configurations 
can give rise to a premature capture but a non-convex configuration cannot evolve from a convex configuration.  Behroozi and Gagnon in \cite{BG1979} generalize \emph{some aspects} of
the proof for the n-bug systems but leave some open questions. Thus, conditions for mutual capture  for n-bug systems, $n > 4$, remain  conjectures supported by simulations, see e.g. \cite{BG1975}.  Richardson shows in \cite{Richardson},  that, in the general case of $n$ bugs in $k$ dimensions, it is \emph{possible} for bugs to capture their prey
without all bugs simultaneously doing so even for non-collinear initial positions, however the probability of a non-mutual capture occurrence is zero.

We note that the simple bugs model is not the only commonly used model for non-linear cyclic pursuit. Another frequently used model for systems in cyclic pursuit is based on higher order agent behaviour, like the unicycle model, describing  wheeled vehicles,  subject to a single non-holonomic constraint, see e.g. \cite{MBF2004}, \cite{Dovrat-TR}, \cite{MBF2004b}, \cite{DK2007}. 

Our contribution is in deriving the emergent behaviour of the "bugs", when an external controller broadcasts a velocity signal which is detected by a \emph{random} set of "bugs" in the group. This problem has not been previously investigated. \cite{DS2014} analyzes the problem of agents, modeled by single integrator dynamic, in bearing-only cyclic pursuit with a \emph{moving target}. 
The moving target can be seen as a leader broadcasting velocity, but this leader is a special purpose agent which does not participate in the cyclic pursuit.   Moreover,  there is a basic assumption that \emph{all} agents in cyclic pursuit detect target's velocity and can  sense the bearing to target. This model is entirely different from our paradigm where each agent can sense the bearing only to its leading agent and the desired velocity is available only to a (random) subset of agents.

\subsection{Paper outline and main results}
The paper is organized as follows:
\begin{itemize}
\item In section \ref{LinCyclic} we derive the emergent behaviour in case of linear cyclic pursuit, with  a broadcast steering control, using properties of linear systems, and show 
that in this case the agents will asymptotically align in the direction of $U_c$ and move asymptotically as a time-independent linear formation with velocity $\displaystyle \frac{n_l}{n}U_c$, where $n_l$ is the number of agents detecting the steering control. In this case there is no restriction on the size of $\|U_c\|$.
\item In section \ref{NonLinCyclic} we derive  the emergent behaviour for the assumed non-linear model (the bugs model with external input detected by a random set of agents) and show that $\|U_c\|$ within  an upper  bound ensures convergence to a moving point. 
In this case, if at least one agent detected the broadcast control,  the agents will all  move, after the mutual capture time, as a single point with velocity $U_c$. 
\end{itemize}
  All the analytically derived results  are illustrated by simulations.

\section{Linear cyclic pursuit with broadcast control}\label{LinCyclic}
In our paradigm, the equation of motion of agents performing linear cyclic pursuit, i.e. sensing relative position to the chased agent, in the presence of a broadcast velocity control, $U_c$, can be written as
\begin{equation}\label{pi-LC1}
\dot{p}_i(t)=p_{i+1}(t)-p_i(t)+b_i(t) U_c(t)
\end{equation}
where
\begin{equation}\label{eq-b_i}
  b_i(t)= \begin{cases}
    1 \quad \text{if agent $i$ detected the external control}\\
    0 \quad \text{otherwise}\\
    \end{cases}
\end{equation}

 Since $p_i=(x_i,y_i)^T$, and $U_c (t)=(U_x(t), U_y(t))^T$,
eq. (\ref{pi-LC1}) evolves independently in the $x$ and $y$ directions, thus we can consider only one component, say $x$. If we aggregate the $x$ position component of all the agents we can write
\begin{equation}\label{x-LC1}
\dot{X}(t)=M X(t)+B(t) U_x(t)
\end{equation}
where
\begin{itemize}
  \item $X=(x_1, x_2,....,x_n)^T$
  \item $M$ is the circulant matrix (\ref{eq-M}) representing the interactions graph of agents in linear cyclic pursuit\\
   \begin{equation}\label{eq-M}
   M  = circ[ -1, 1, 0, 0, \hdots, 0]=
          \begin{bmatrix}
           -1 & 1 & 0 & 0 & \hdots & 0 \\
            0 & -1 & 1 & 0 & \hdots & 0 \\
	   &&\vdots \\
            1 & 0  && \hdots& 0  & -1 \\
          \end{bmatrix}
  \end{equation}

  \item $B(t)$ is the leaders indicator vector at time $t$, i.e. the i'th entry in the vector $B(t)$ is $b_i(t)$ and $b_i(t)$ is defined by (\ref{eq-b_i})
\end{itemize}

  In eq. (\ref{x-LC1}), $M$ is time independent and if we assume  $B(t)$ and $U_x(t)$ to be  piecewise constant, i.e. we assume that the time-line can be divided into intervals, $t \in [t_k, t_{k+1})$, where $B(t) \overset{\Delta}{=} B(t_k)$, and $U_x(t) \overset{\Delta}{=} U_x(t_k)$, $t_k$ is the time of change of leaders or of the broadcast control.  In the sequel we treat each time interval, $[t_k,t_{k+1} )$, separately. Thus, it is convenient to suppress the subscript $k$.  Moreover, it is convenient to denote by $t$  the \emph{relative time since the beginning of the interval} ($t=0$) and  by $x(0)$  the state of the system at this time, by $B$ the leaders indicator during the interval and by $U_x$ the exogenous control during the interval. In each time interval eq. (\ref{x-LC1}) evolves as a linear time independent system, which has the well known solution (cf. \cite{TK})

\begin{equation}\label{eq-pieceDyn}
x(t) = e^{M t} x(0) + \int_{0}^{t} e^{M(t-\tau)} B U_x \mathrm{d}\tau
\end{equation}

We note that $M$ defined by (\ref{eq-M}) is a normal matrix, (Appendix \ref{Normal}), therefore it is unitarily diagonizable, i.e. $M = V \Lambda V^*$, where $V$ is a unitary matrix of eigenvectors, $V^*$ denotes the transpose of the complex conjugate of $V$ and $\Lambda$ is a diagonal matrix of eigenvalues. 

\textbf{Remark}: Due to the circulant structure of $M$, $V$ is the DFT matrix 

In particular, since $M$ is the circulant matrix with $c_0=-1$, $c_1=1$ and $c_k=0; \quad k=2, 3, ...,n-1$ , we have (from  Appendix \ref{circulant}) 
\begin{itemize}
\item the eigenvalues of $M$ can be written as 
\begin{equation}\label{lambda_k}
 \lambda_k=\rho^k-1; \quad k=0,....,n-1
\end{equation}
where  $\displaystyle  \rho \overset{\Delta}{=} e^{-2\pi j/n}$
\item with corresponding eigenvectors 
\begin{equation}\label{v_k}
 v_k=\frac{1}{\sqrt{n}} \left ( 1,\rho^k, \rho^{2k},\hdots, \rho^{k(n-1)}  \right)^T; \quad k \in \left \{0, 1, \hdots, n-1 \right \}
\end{equation}
\end{itemize}

It is easily seen that $\lambda_0=0$, with a corresponding eigenvector $v_0 = \frac{1}{\sqrt{n}} \mathbf{1}_n$, while the remaining eigenvalues, $\lambda_k; \quad k=1,....,n-1 $,  have negative real part.

These features of $M$ constitute the basis for the derivation of the emergent behaviour of swarms in linear cyclic pursuit with broadcast control.

Following the methodology in \cite{S-B-arxiv}, we write eq. (\ref{eq-pieceDyn})  as
\begin{equation}\label{eq-xdyn2}
X(t) = x^{(h)}(t) + x^{(u)}(t)
\end{equation}
where
 \begin{itemize}
   \item $x^{(h)}(t)= e^{M t} X(0)$ represents the zero input, homogeneous, solution
   \item $x^{(u)}(t)=\int_{0}^{t} e^{M (t-\tau)} B U_x \mathrm{d}\tau$ represents the contribution of the exogenous input (broadcast control) to the group dynamics
 \end{itemize}

\subsection{Homogeneous linear cyclic pursuit}\label{H_cyclic}

\begin{equation}\label{eq_xh}
  x^{(h)}(t)= e^{M t} x(0) = V e^{\Lambda t} V^* x(0) = v_0 v_0^T X(0) + \sum_{k=1}^{n-1} v_k e^{\lambda_k t} v_k^*X(0)
\end{equation}
Since $\Re (\lambda_k) <0 $ for $k=1,....,n-1$, we have
\begin{eqnarray*}
 \sum_{k=1}^{n-1} v_k e^{\lambda_k t} v_k^*X(0) & \xrightarrow[t \rightarrow \infty] { } & \mathbf{0}_n\\
 x^{(h)}(t) & \xrightarrow[t \rightarrow \infty] { } &  v_0 v_0^T X(0)= \frac{1}{n} \sum_{i=1}^n x_i(0)\mathbf{1}_n
\end{eqnarray*}

Thus, a homogeneous system of $n$ agents performing linear cyclic pursuit will asymptotically converge to a point, the centroid, determined by the initial conditions, which is a well known result, see e.g. \cite{Bruck}. However, the derivation in this section introduces a methodology which will be useful in the sequel.

\subsection{The effect of the exogenous control}\label{ControlContribution}
Recalling that we consider a time interval where $B, U_x$ are constant, we obtain
 \begin{equation}\label{gen-xu}
   x^{(u)}(t)= \int_{0}^t e^{M(t-\tau)}B U_x d\tau = \int_0^{t} e^{M\nu}B U_x d\nu = \left [ \int_0^{t} e^{M\nu}d\nu    \right ]B U_x
\end{equation}
Using the diagonalization of $M$, as in subsection \ref{H_cyclic}, we have
\begin{itemize}
  \item $M=V \Lambda V^*$, where $V$ is a unitary matrix formed by the eigenvectors of $M$, given by (\ref{v_k}), and $\Lambda$ is a diagonal matrix of the corresponding eigenvalues of $M$, given by (\ref{lambda_k}) 
  \item $ e^{M \nu}= V e^{\Lambda \nu} V^* = \sum_{k=0}^{n-1}e^{\lambda_k \nu} [v_k v_k^*]$

\end{itemize}

Since $M$ has  a single zero eigenvalue and the remaining eigenvalues have negative real part,  we can decompose  $x^{(u)}(t)$  in two parts:
\begin{equation}
  x^{(u)}(t)= x^{(a)}(t)+ x^{(b)}(t)
\end{equation}\label{eq-xu}
where
\begin{itemize}
  \item $x^{(a)}(t)$ is the zero eigenvalue dependent term, representing the movement in the agreement space
  \item $x^{(b)}(t)$ is the remainder, representing the deviation from the agreement space
\end{itemize}

\subsubsection{Movement in the agreement space}
 \begin{equation}\label{xua}
    x^{(a)}(t) = \int_0^t e^{\lambda_0  \nu} \text{} v_0 v^*_0 B U_x d\nu = v_0 v^*_0 B U_x  t = \frac{n_l}{n} \text{} U_x \text{} t\mathbf{1}_n
  \end{equation}
where $\lambda_0=0$, $n_l$ is the number of leaders in the considered time interval and  $\displaystyle v_0=\frac{1}{\sqrt{n}} \mathbf{1}_n$,  $\displaystyle v^*_0=\frac{1}{\sqrt{n}} \mathbf{1}_n^T$ and $\mathbf{1}_n^T B=n_l$.
Recalling that eq. (\ref{xua}) holds also for the $y$ axis, with the corresponding component of $U_c$,  i.e.

\begin{equation}\label{yua}
    y^{(a)}(t) = \frac{n_l}{n} \text{} U_y \text{} t\mathbf{1}_n
  \end{equation}
  we have

 \begin{lem}\label{L-ua}
    A group of $n$ agents performing linear cyclic pursuit, with $n_l$ agents receiving an exogenous velocity vector control $U_c$,  will asymptotically align in the direction of the vector $U_c$ and move with a common speed that is proportional to the ratio of $n_l$ to $n$, i.e. $\displaystyle \frac{n_l}{n} U_c $.
\end{lem}

\subsubsection{Deviation from the agreement space}\label{Dev}
Consider now the remainder $x^{(b)}(t)$ of the input-related part, i.e. the part of $x^{(u)}(t)$ containing all eigenvalues of $M$ other than the zero eigenvalue and representing the agents' state deviation from the agreement subspace.

We have

\begin{eqnarray}
  x^{(b)}(t)& = &
 \left [ \sum_{i=1}^{n-1} \int_0^t \left ( e^{\lambda_i  \nu} \right ) v_i v_i^* d\nu \right ] B U_x\\
 & = & \left[ \sum_{i=1}^{n-1} \frac{1}{\lambda^i}(1-e^{\lambda_i t})v_i v^*_i \right ] B U_x \label{xub}
\end{eqnarray}
Since all eigenvalues $\lambda_i \text{ for } i = 1, \dots, n-1$ have strictly negative real parts, $x^{(b)}(t)$ converges  asymptotically to a time independent vector, denoted by $\xi_x$, given by:
\begin{equation}\label{Delta}
  \xi_x \triangleq x^{(b)}(t\rightarrow \infty)= \left[ \sum_{i=1}^{n-1} \frac{1}{\lambda_i}v_i v^*_i \right ] B U_x
\end{equation}
Similarly, for the $y$ axis,
\begin{equation}\label{Delta-y}
  \xi_y \triangleq y^{(b)}(t\rightarrow \infty)= \left[ \sum_{i=1}^{n-1} \frac{1}{\lambda_i}v_i v^*_i \right ] B U_y
\end{equation}
and all the following properties of $\xi_x$ hold also for $\xi_y$.
Note that $M$ depends only on the number of agents, thus for a given number of agents and a constant broadcast control, $U_c$, the deviations depend only on $B$, i.e. on the agents detecting the broadcast control. Moreover, if all the agents receive the broadcast control then there are no deviations, i.e. the agents converge and move as a single point.
 \begin{lem}\label{L-dev}
    A group of $n$ agents performing linear cyclic pursuit, \emph{with all agents receiving an exogenous velocity control} $U_c$,  will asymptotically  move as a single point with velocity $U_c$.
\end{lem}

\begin{proof}
If all the agents receive the broadcast control then $n_l=n$ and $B=\mathbf{1}_n$. According to Lemma \ref{L-ua}, all the agents will move with a velocity $U_c$. It remains to show that there is no deviation.  Recalling that $v_i^*$ is a left eigenvector of $M$ with eigenvalue $\lambda_i$ and $B=\mathbf{1}_n$ we can rewrite eq. (\ref{Delta}) as
\begin{equation}\label{Delta-2}
  \xi_x = \left[ \sum_{i=1}^{n-1} \frac{1}{\lambda_i}v_i \frac{1}{\lambda_i} v^*_i M \right ] \mathbf{1}_n U_x = \mathbf{0}_n
\end{equation}
and similarly $\displaystyle \xi_y = \mathbf{0}_n$
\end{proof}

\subsection{Asymptotic trajectories of $n$ agents with an exogenous velocity control detected by $n_l$ agents}
This section summarizes the results derived in sections \ref{H_cyclic} - \ref{Dev}.

The  asymptotic position of agent $i$, chasing agent $i+1$, in the two-dimensional space, when an external control  $U_c=(U_x \quad U_y)^T$  is detected by $n_l$ agents,  will be

\begin{equation}\label{pAsymp}
  p_i(t \rightarrow \infty)=  \left [
 \alpha +\beta  U_c  t + \gamma_i  U_c \right ]
  \end{equation}

 where
 \begin{itemize}
   \item $\alpha=(\alpha_x\text{   }\alpha_y)^T=\frac{1}{n}\sum_{i=1}^{n}p_i(0)$ is the agreement, or gathering, point when there is no external input
   \item $\displaystyle \beta = \frac{n_l}{n}$ and $\beta  U_c$ is the collective velocity.
   \item $\alpha + \beta U_c t$ is the position of the moving agreement point at time $t$
   \item  $\gamma_i U_c $ is the deviation of agent $i$ from the moving agreement point, where
   \begin{equation*}
  \gamma= \left[ \sum_{i=1}^{n-1} \frac{1}{\lambda_i}v_i v^*_i \right ] B
\end{equation*}
and $\gamma_i$ is the $i'th$ element of $\gamma$
 \end{itemize}
 
 \subsection{Illustration of linear cyclic pursuit - single interval}
 We illustrate the derived analytical results by 2 simulated examples. Both examples assume six agents, starting from the same random positions, shown in Fig. \ref{fig:InitTopo6}, but differing in the broadcast control and the set of agents detecting it. Example1 and Example2 were run for 50 secs (500000 points, dt=0.0001). This simulation time was long enough to obtain the analytically derived asymptotic behaviour.
 
\begin{figure}[H]

\centering
\includegraphics[scale=0.5]{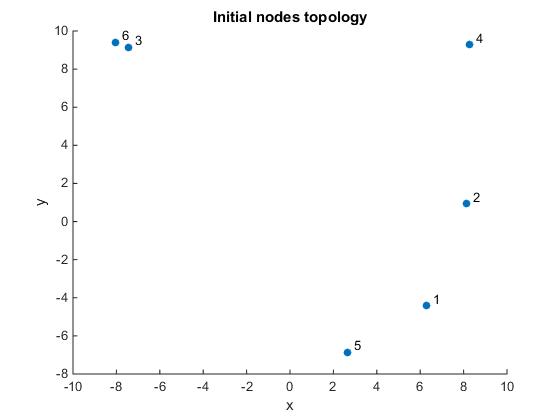}
\caption{Initial positions}
\label{fig:InitTopo6}

\end{figure}

\begin{enumerate}
\item \textbf{Example1}
\begin{itemize}

\item broadcast control, $U_c=(5,1)$ (Slope=0.2)
\item set of ad-hoc leaders $\{0,1,0,0,0,0 \}$, thus $n_l=1$ (out of 6)
\end{itemize}
\textbf{Simulated results for Example1:}

\begin{figure}[H] 
\centering
\includegraphics[scale=0.5]{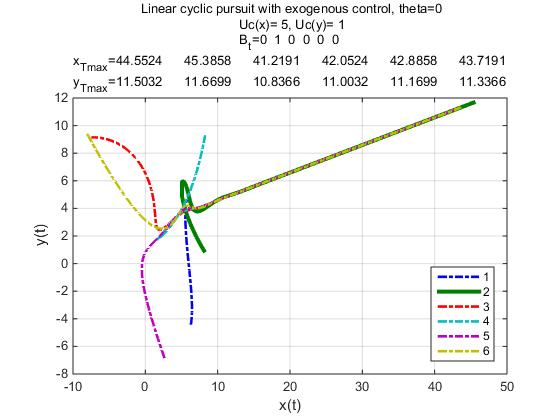}
\caption{ Emergent trajectories }\label{fig:Lin6Traj}
\end{figure}
A solid line represents the leader while the trajectories of followers are shown by dotted lines.

\begin{figure}[H]
\centering
\includegraphics[scale=0.5]{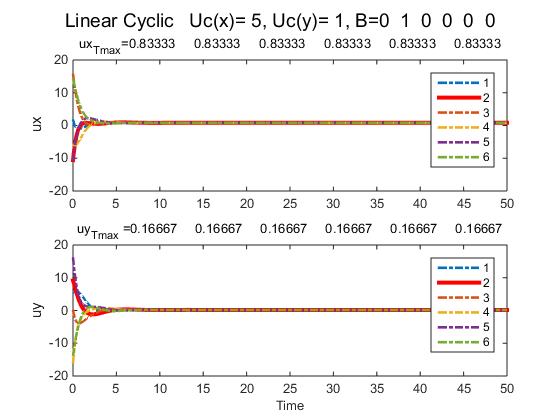}
\caption{ Agents velocities in a long time interval}
\label{fig:Lin6Vel}

\end{figure}

In this figure a solid red line represents a leader while  dotted lines represent followers.
The velocities converge to  $ux_{Tmax}, ux_{Tmax}$ which agree with the derived asymptotic velocities of $\displaystyle \frac{n_l}{n} U_c$, where $n_l=1, n=6$.

\begin{figure}[H]
\centering
\includegraphics[scale=0.5]{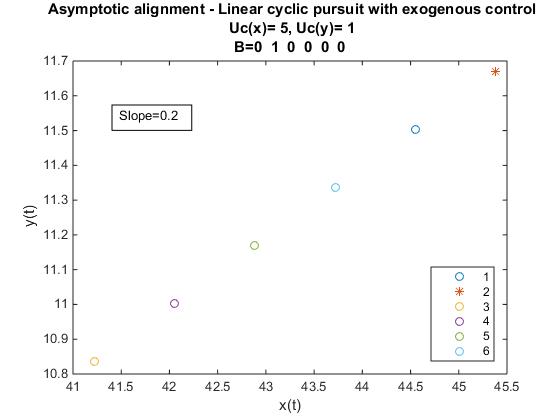}
\caption{ "Asymptotic" agents alignment}
\label{fig:Lin6Align}
\end{figure}

This figure shows the position of the agents at the end of the time interval. A star represents a leader while o represents followers. The slope of the line equals the slope of $U_c$, i.e. the agents align in the direction of $U_c$.
\item \textbf{Example2}
\begin{itemize}

\item broadcast control, $U_c=(6,3)$ (Slope=0.5)
\item set of ad-hoc leaders $\{1,1,0,1,1,1 \}$, thus $n_l=5$ (out of 6)
\end{itemize}

\textbf{Simulated results for Example2:}

\begin{figure}[H]
\centering
\includegraphics[scale=0.5]{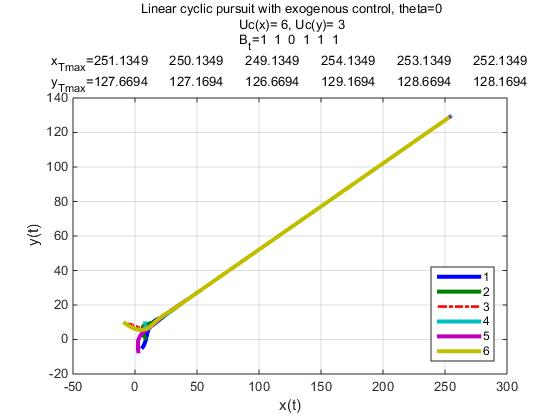}
\caption{ Emergent trajectories }\label{fig:Ex2-Lin6Traj}
\end{figure}

\begin{figure}[H]
\centering
\includegraphics[scale=0.5]{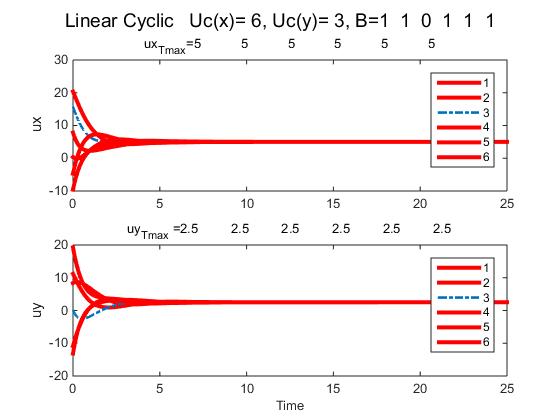}
\caption{ Agents velocities converge to asymptotic values} 
\label{fig:Ex2-Lin6Vel}

\end{figure}
Displayed $ux_{Tmax}, ux_{Tmax}$ agree with the derived asymptotic velocities of $\displaystyle \frac{n_l}{n} U_c$, where $n_l=5, n=6$

\begin{figure}[H]
\centering
\includegraphics[scale=0.5]{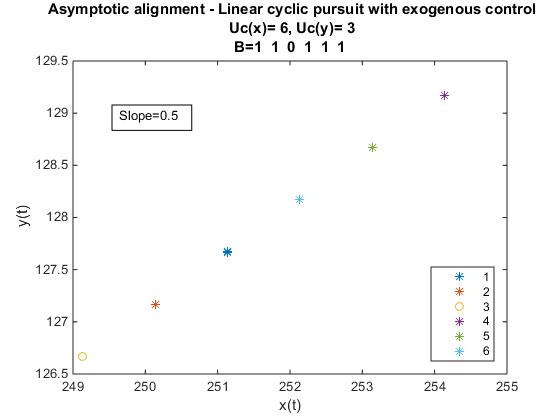}
\caption{ Agents asymptotically align along the direction of $U_c$}
\label{fig:Ex2-Lin6Align}
\end{figure}

\end{enumerate}

We observe that in both examples the agents \textbf{asymptotically} align in the direction of $U_c$ and move as a linear formation with velocity $\displaystyle \frac{n_l}{n}U_c$, as expected.

We emphasize that this behaviour is indeed asymptotic, not obtained in a short time interval, as shown in figures \ref{fig:Ex2-short-Lin6Vel} and \ref{fig:Ex2-short-Lin6Align}, next.

\begin{figure}[H]
\centering
\includegraphics[scale=0.5]{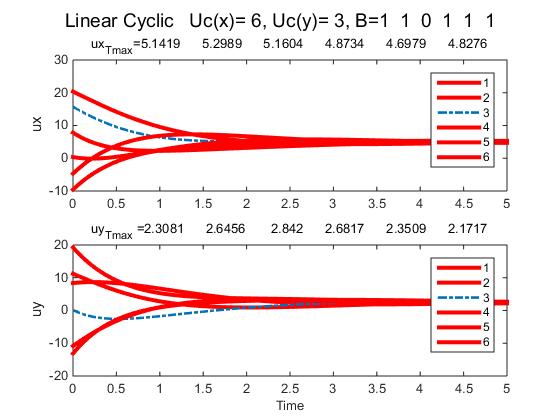}
\caption{Agents velocities in a short time interval (5sec)}
\label{fig:Ex2-short-Lin6Vel}

\end{figure}

\begin{figure}[H]
\centering
\includegraphics[scale=0.5]{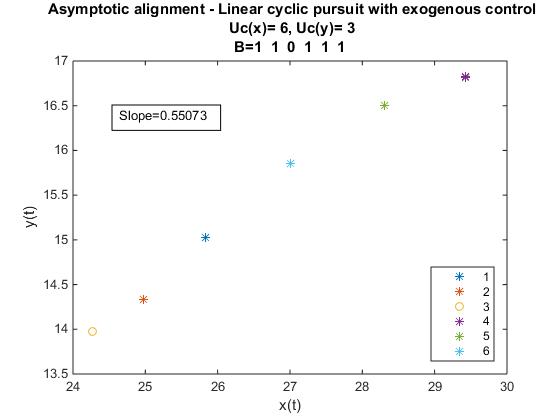}
\caption{ Agents positions ("alignment") at the end of a short interval (5sec) }
\label{fig:Ex2-short-Lin6Align}
\end{figure}

\subsection{Illustration of linear cyclic pursuit - multiple intervals}
We recall that we assumed $B(t)$ and $U_c(t)$ to be  piecewise constant,  i.e. $B(t) \overset{\Delta}{=} B(t_k)$, and $U_c(t) \overset{\Delta}{=} U_c(t_k)$, where $t_k$ is the time of change of the set of ad-hoc leaders or of the broadcast control, and we treated separately each time interval,
 $[t_k,t_{k+1} )$. In the above sections we considered a single time interval, where $U_c$ and $B$ are constant.  

 In this section we show, by simulation, the emergent behaviour over multiple time intervals. We consider as before 6 agents, starting at initial positions as illustrated in Fig. \ref{fig:InitTopo6} but, starting with $U_c(0)=(6,3)^T$, we allow for discrete changes in $U_c$ every 10 secs.  $U_c(t)$ is shown in Fig. \ref{fig:MultipleUc}.  In this example the set of ad-hoc leaders is randomly selected at $t=0$ and remains constant afterwards, i.e. $B(t) = B(0) \forall t$.

 The emergent trajectories are illustrated in  Fig. \ref{fig:TrajMultipleUc}.
\begin{figure}[H]
\centering
\includegraphics[scale=0.5]{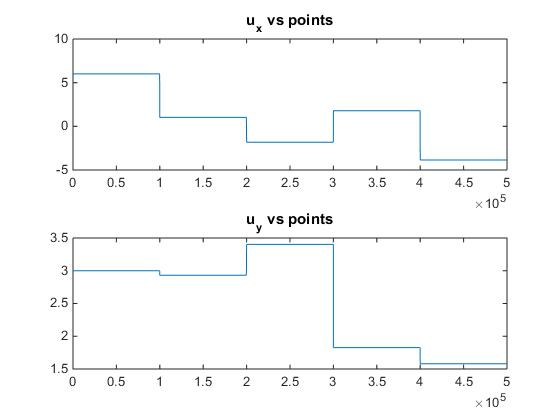}
\caption{$U_c(t)$ profile }
\label{fig:MultipleUc}
\end{figure}

\begin{figure}[H]
\centering
\includegraphics[scale=0.5]{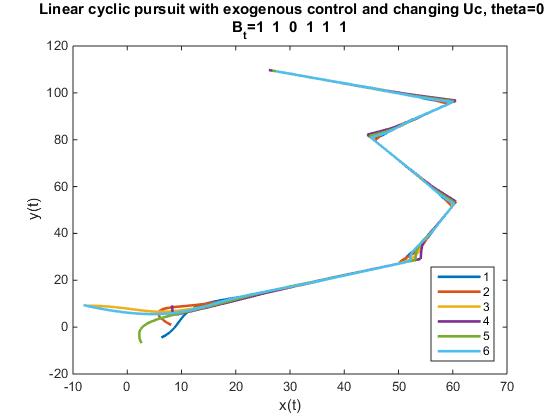}
\caption{Emerging trajectories with $U_c(t)$ as in Fig. \ref{fig:MultipleUc} }
\label{fig:TrajMultipleUc}
\end{figure}

Fig. \ref{fig:VelMultipleUc} shows the agents velocities vs. analytically computed asymptotic velocity of the linear formation emerging from a broadcast velocity signal as shown in  Fig. \ref{fig:MultipleUc}.

\begin{figure}[H]
\centering
\includegraphics[scale=0.5]{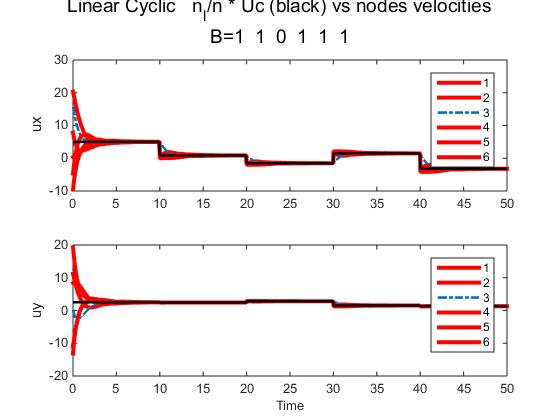}
\caption{Emerging velocities vs computed formation velocity }
\label{fig:VelMultipleUc}
\end{figure}

We note that changes in the random set of ad-hoc leaders, within an interval where $U_c$ is constant, will affect the speed of the agents (if the \emph{number} of ad-hoc leaders changes) and also the arrangement of the agents within the linear formation.

\newpage
\section{Non-linear cyclic pursuit }\label{NonLinCyclic}
We assume the model for non-linear cyclic pursuit to be the "bugs" model, where agent $i$ chases agent $i+1$  (agent $n$ chases agent $1$) with constant, common, speed along the line of sight,  merges with it upon capture and the two agents continue with velocity $\dot{p}_{i+1}$. Upon capture, the number of agents is reduced.  Agent $i$ is said to capture agent $i+1$ if the distance between them, $d_i$, is zero
\begin{equation}\label{eq-d_i}
d_i(t)=\|p_{i+1}(t) - p_i(t)\|
\end{equation} 
where $p_i(t)$ is the position of agent at time $t$ and $\|.\|$ represents the Euclidean norm. 

In case of an external broadcast velocity control, upon capture, the merged agent will be  assumed to detect the broadcast velocity signal, if either one or both of the merged agents detected the broadcast velocity signal.

We first analyze the properties of the "bugs" model, without external broadcast velocity signal,    
 and then derive the impact of the broadcast velocity which is  detected by a random set of agents.
\subsection{"Bugs" model}\label{Bugs} 
Let $d_i(t)$ be defined by eq. (\ref{eq-d_i}) and assume, without loss of generality, the speed of all agents to be 1. Then, the pursuit is formally defined as follows:

\begin{eqnarray}  
 \text{ if    } d_i(t) \neq 0  \quad \dot{p}_i(t) & = & \frac{p_{i+1}(t)-p_i(t)}{d_i(t)}\label{pi-dot}\\
 \text{ if    } d_i(\hat{t}) = 0  \quad p_i(t) & = & p_{i+1}(t) \text{  for all  } t \geq \hat{t}\label{d_i-0}
 \end{eqnarray} 
  
 \begin{lem}\label{BugsProperties}
 If $d_i(t) \neq 0$ for all $i$ and the motion of each agent is represented by (\ref{pi-dot}), then 
\begin{enumerate}[label=(\alph*)]
\item $d_i(t)$ is monotonically non-increasing for all $i$.
\item There exists a finite time $T_m$ such that $d_i(T_m)=0$ for all $i$  
\end{enumerate}
 \end{lem} 
 \begin{proof}
 In the sequel, for simplicity of notation, we omit explicit reference to time $t$, whenever it is not confusing.
 \begin{enumerate}[label=(\alph*)]
 \item :  
 $d_i(t)$ is non-increasing iff  $\dot{d}_i(t) \leq 0$ for all $t$.

 From (\ref{eq-d_i}) we have 
\begin{eqnarray}
d_i^2 & = & \langle p_{i+1}-p_i, p_{i+1}-p_i\rangle\label{d_i-sqr}\\
\dot{d}_i & = & \frac{1}{d_i} \langle \dot{p}_{i+1}-\dot{p}_i ,p_{i+1}-p_i\rangle\label{gen-dot-d_i}\\
         & = & \langle \dot{p}_{i+1}-\dot{p}_i ,\dot{p}_i\rangle\\
         & = & \cos(\theta) -1 \leq 0 \quad \text{  for any   } \theta 
\end{eqnarray}
where we used $\|\dot{p}_i\| =\|\dot{p}_{i+1}\|= 1$ and $\theta $ is defined as in Fig. \ref{fig-Theta}.
\begin{figure}[H]
\begin{center}
\includegraphics[scale=0.5]{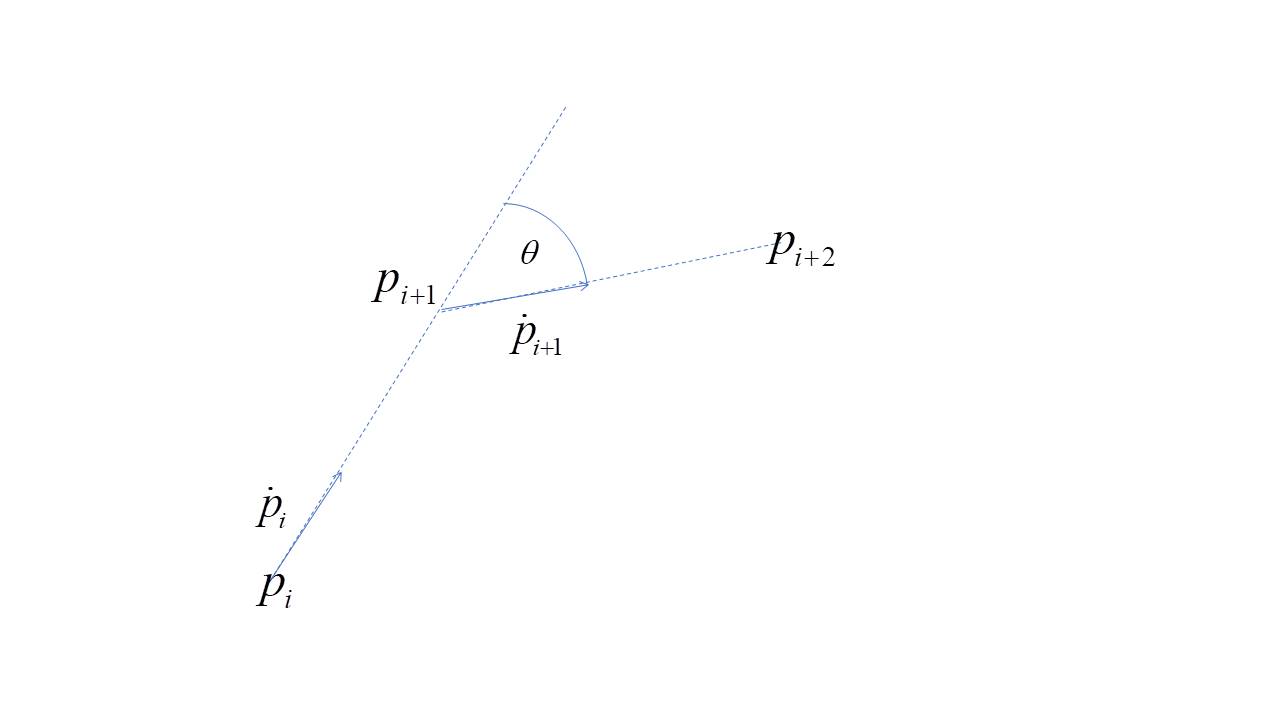}
\caption{$\theta$ is the angle between vectors $\dot{p}_i$ and $\dot{p}_{i+1}$ }\label{fig-Theta}
\end{center}
\end{figure}

 \item :
Richardson shows in \cite{Richardson}, Lemma 1.1, that $T_m$ exists and satisfies $T_m \leq t_0 +2 n \sum_{i=1}^n d_i(t_0)$. The proof is repeated in Appendix \ref{proof_Tm}, for completeness.   
 \end{enumerate}  
\end{proof}

\subsection{"Bugs" model with external input}\label{BugsInp}  
Let $K_i$ denote the set of $k$ indices of 
$K_i=\{i-k, i-k+1, \dots, i-1\}$.  If $d_j = 0 \quad \forall j \in K_i$ and   $d_i \neq 0$ we say that the $k$ agents in $K_i$ collapsed into agent $i$.

In order to analyze the motion of agent $i$ performing bearing-only cyclic pursuit when a  velocity signal, $U_c$, is broadcast by an external controller and detected by a random set of agents, we have to consider  three different scenarios :
  \begin{enumerate}
  \item  $i$ is a free agent, i.e. $d_i \neq 0$ and $K_i =\{ \}$
  \item $k$ agents labeled $(i-k, i-k+1, \dots, i-1)$ have merged into $i$, i.e.  $d_i \neq 0 $ and  $ d_j =0 \text{  }\forall j \in K_i$  
  \item Agent $i$ collapsed into agent $j$, i.e. $d_j \neq 0, \quad d_i=0$ and $i \in K_j$  
  \end{enumerate}

  These cases are formalized by eqs. (\ref{pi-dot-Uc}) - (\ref{pi-follower}):
 \begin{equation}\label{pi-dot-Uc}
 \dot{p}_i = 
 \begin{cases}
    \frac{p_{i+1}-p_i}{d_i}+b_i U_c \quad \text{ if    } d_i \neq 0 \text{   and  } K_i=\{ \}\\
    \frac{p_{i+1}-p_i}{d_i}+(\bigvee_{j \in K_i} b_j) U_c \quad \text{ if    } d_i \neq 0 \quad \forall j \in K_i\\
 \end{cases}
 \end{equation}  
\begin{equation}
  p_i =  p_j \quad \forall i \in K_j \label{pi-follower}
\end{equation}  
 
where $\bigvee$ stands for logical "or",  
\begin{equation}\label{bi}
  b_i = \begin{cases}1 ; \text{    if    } i \in N^l \\
  0 ; \text{   otherwise}
  \end{cases}
\end{equation}
and $N^l$ is the set of agents that detected the broadcast control.

If $d_i =0 \quad \forall i$ then    
$$  \dot{p}_i =   (\bigvee_{ \forall i}  b_i) U_c  \quad \forall i$$
Thus, once gathered, the agents will move as a single agent with velocity $U_c$, if at least one of the agents detected the external control $U_c$.

In section \ref{SingleDist}, we show the impact of the exogenous velocity signal, $U_c$ on a single distance,  $d_i(t)$, defined by (\ref{eq-d_i}), when either $i$ or $i+1$, or both, detect $U_c$.  We derive the behaviour of $d_i$, in each case, as a function of $\|U_c\|$ and the \emph{instantaneous}   geometry.    Since the agents are mobile, the geometry is time dependent. Therefore, we cannot deduce from $d_i(t)$ instantaneously decreasing that it will decrease for all $t$ or that it will reach zero.    In section \ref{UcBounds} we derive an upper bound on  $\|U_c\|$ that ensures convergence to a point in finite time, $T_m$, i.e. $\sum_{i=1}^n d_i(T_m) =0$, independently of the instantaneous geometries.

\subsubsection{Single $d_i$ behaviour }\label{SingleDist}
Let $d_i > 0 \quad \forall i$. This is a valid assumption since subsequent to a collision (capture) the system evolves as a cyclic pursuit with fewer agents. Thus, 
\begin{eqnarray}
\dot{d}_i & = & \frac{1}{d_i} \langle \dot{p}_{i+1}-\dot{p}_i ,p_{i+1}-p_i\rangle\label{gen-dot-d_i -2}\\
\dot{p}_i & =& \frac{p_{i+1}-p_i}{d_i} +b_i U_c
\end{eqnarray}
 The impact of $U_c$ on any distance $d_i > 0$ can be separated into four cases:
\begin{enumerate}
  \item Neither $i$ nor $i+1$ detected the signal $U_c$  $\Rightarrow b_i = 0, \quad b_{i+1}=0$
  \item Both $i$ and $i+1$ detected the signal $U_c$  $\Rightarrow b_i = 1, \quad b_{i+1}=1$
  \item $i$ detected the signal $U_c$, but $i+1$ did not  $\Rightarrow b_i = 1, \quad b_{i+1}=0$
  \item $i+1$ detected the signal $U_c$, but $i$ did not  $\Rightarrow b_i = 0, \quad b_{i+1}=1$
  \end{enumerate}
  \paragraph{Case 1: $b_i=0, \quad b_{i+1}=0$}.\\
  
This case is identical to Lemma \ref{BugsProperties}(a), Fig. \ref{fig-Theta}, and therefore  $\dot{d}_i = \cos(\theta) -1 \leq 0$, independently of the broadcast $U_c$.
  \paragraph{Case 2: $b_i=1,\quad b_{i+1}=1$}.\\

In this case $\displaystyle  \frac{p_{i+1}-p_i}{d_i}  =  \dot{p}_i -U_c$ and (\ref{gen-dot-d_i -2}) can be rewritten as
\begin{eqnarray*}
\dot{d}_i & = & \langle \dot{p}_{i+1}-\dot{p}_i ,\dot{p}_i-U_c\rangle\\
          & = & \langle (\dot{p}_{i+1}-U_c)-(\dot{p}_i-U_c) ,\dot{p}_i-U_c\rangle 
\end{eqnarray*}
Recalling that $\displaystyle (\dot{p}_j-U_c)= \frac{p_{j+1}-p_j}{d_j}; \quad j=i, i+1$ we have 
$\|\dot{p}_j-U_c\|=1; \quad j=i, i+1$  and thus 
\begin{equation}
\dot{d}_i = \cos(\theta) -1
\end{equation}
where $\theta$ is now the angle between $ \dot{p}_i-U_c $ and $ \dot{p}_{i+1}-U_c $, see Fig. \ref{fig-Uc-both}.
Thus, in this case $\dot{d}_i \leq 0$ , for any $ U_c $, same as case 1. 
\begin{figure}[H]
\begin{center}
\includegraphics[scale=0.8]{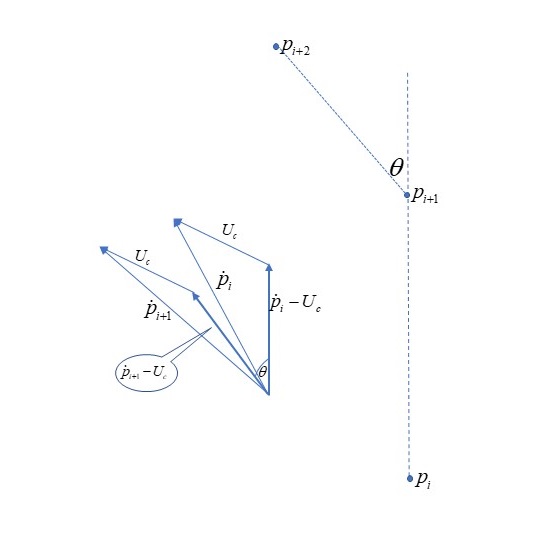}
\caption{$U_c$ detected by both $i$ and $i+1$ }\label{fig-Uc-both}
\end{center}
\end{figure}

 \paragraph{Case 3: $b_i=1,\quad b_{i+1}=0$}.\\
 
In this case $\dot{p}_{i+1}$ is not affected by $U_c$, as shown in Fig. \ref{fig-Uc-i}, where we have  
 $\|\dot{p}_i-U_c\|=1$, and  $\|\dot{p}_{i+1}\|=1$.

\begin{figure}[H]
\begin{center}
\includegraphics[scale=0.8]{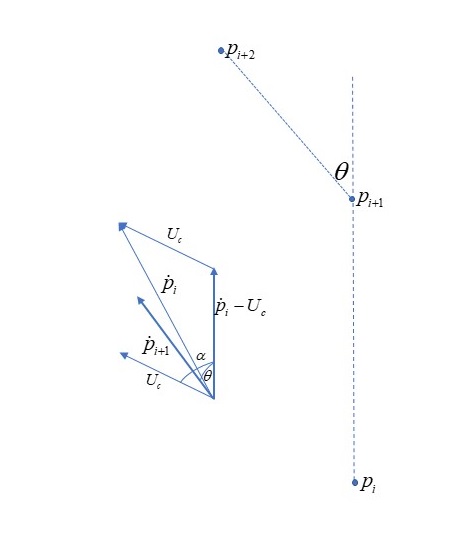}
\caption{$U_c$ detected by $i$ but not by $i+1$.}\label{fig-Uc-i}
\end{center}
\end{figure}

In this case
\begin{eqnarray*}
\dot{d}_i & =  & \langle \dot{p}_{i+1}-\dot{p}_i ,\dot{p}_i-U_c\rangle\\
         & = & \langle (\dot{p}_{i+1}-U_c)-(\dot{p}_i-U_c) ,\dot{p}_i-U_c\rangle\\
         & = &  \langle (\dot{p}_{i+1}-U_c) ,(\dot{p}_i-U_c)\rangle -1\\
         & = &  \langle (\dot{p}_{i+1}) ,(\dot{p}_i-U_c)\rangle -  \langle U_c ,(\dot{p}_i-U_c)\rangle-1\\
         & = & \cos(\theta) - \|U_c\| \cos(\alpha) -1
\end{eqnarray*}
where the angles $\theta, \alpha$ are defined as in Fig. \ref{fig-Uc-i}. In order to have $\dot{d}_i \leq 0$ we must have $ - \|U_c\| \cos(\alpha) \leq 1- \cos(\theta)$. Therefore, $d_i$ will be non-increasing
\begin{itemize}
\item for any $\|U_c\|$  if $ \cos(\alpha) \geq 0 $, i.e. $\alpha$ is an acute angle
\item for  $\displaystyle \|U_c\| < \frac{1-\cos(\theta)}{-\cos(\alpha)} $  if $ \cos(\alpha) < 0 $, i.e. $\alpha$ is an obtuse angle
\end{itemize}

 \paragraph{Case 4: $b_i=0,\quad b_{i+1}=1$}.\\
 
In this case $\dot{p}_{i}$ is not affected by $U_c$, as shown in Fig. \ref{fig-Uc-i2}, where we have 
  $\|\dot{p}_i\|=1$, and  $\|\dot{p}_{i+1} - U_c\|=1$.

\begin{figure}[H]
\begin{center}
\includegraphics[scale=0.8]{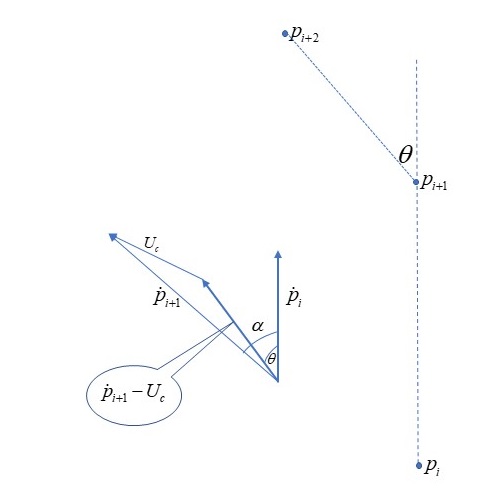}
\caption{$U_c$ detected by $i+1$ but not by $i$.}\label{fig-Uc-i2}
\end{center}
\end{figure} 

\begin{eqnarray*} 
\dot{d}_i & =  & \langle \dot{p}_{i+1}-\dot{p}_i ,\dot{p}_i\rangle\\
         & = &  \langle \dot{p}_{i+1} ,\dot{p}_i\rangle -1\\
         & = &  \langle (\dot{p}_{i+1}-U_c) ,\dot{p}_i\rangle +  \langle U_c ,\dot{p}_i\rangle-1\\
         & = & \cos(\theta) + \|U_c\| \cos(\alpha) -1
\end{eqnarray*}
Therefore, in this case,  $d_i$ will be non-increasing
\begin{itemize}
\item for any $\|U_c\|$  if $ \cos(\alpha) < 0 $, i.e. $\alpha$ is an obtuse angle
\item for  $\displaystyle \|U_c\| < \frac{1-\cos(\theta)}{\cos(\alpha)} $  if $ \cos(\alpha) >  0 $, i.e. $\alpha$ is an acute angle
\end{itemize}

\subsubsection{Gathering in finite time }\label{UcBounds}
In section \ref{SingleDist} we derived the \emph{instantaneous} behaviour of a single inter-agents distance when none or both or one of the limiting agents detected the broadcast control. Since the geometry of the agents is time dependent we cannot deduce the emergent behaviour of the system from the instantaneous behaviour. Given that the agents behave according to the "bugs" model with broadcast control, see (\ref{pi-dot-Uc}), we want to show that there exist conditions  that ensure the gathering of the agents to a moving point in finite time. These conditions are derived in Theorem \ref{sum_di_dot-1} as an upper bound on the magnitude of 
$U_c$. 
\begin{thm}\label{sum_di_dot-1}
If $d_i$ is defined by (\ref{eq-d_i}) and  $d_i(t) \neq 0; i=1,...,n $,   then 
\begin{enumerate}[label=(\alph*)]
\item if 
${\displaystyle \|U_c\| \leq \frac{1}{2n^2}}$, there exists a finite time $T_m$ such that 
${\sum_{i=1}^n d_i(t \geq T_m)=0}$
\item (a) holds for all times when $\sum d_i(t) \neq 0$ even if $d_i(t) = 0$ for some $i$, i.e. the capture is non-mutual
\item $\displaystyle T_m \leq t_0 +\frac{2 n \sum d_i(t_0)}{1- 2 n^2 \|U_c\|}$, where $t_0$ is the initial time.
\end{enumerate}
\end{thm}

\begin{proof}
(a):

Following the methodology in \cite{Richardson}, we prove that for $\displaystyle \|U_c\| < \frac{1}{2n^2} $ we have 
 $ \sum_{i=1}^n\dot{d}_i \leq -c $, where  $\displaystyle c=\frac{1}{2n}-n \|U_c\|>0$, therefore there exists a time $T_m$ such that 
 $\sum_{i=1}^n d_i(T_m)=0$.

For $d_i(t) \neq 0 $ we have
\begin{eqnarray}
  \dot{p}_i &=& u_i+b_i U_c \label{dot_p_i_Uc}\\
  u_i &=& \frac{p_{i+1}-p_i}{d_i} \label{u_i}
\end{eqnarray}
Using $d_i^2 = \|p_{i+1}-p_i \|^2$, (\ref{u_i}) and  (\ref{dot_p_i_Uc}),  we have
\begin{eqnarray*}
   \dot{d}_i &=& u_i^T (\dot{p}_{i+1}-\dot{p}_i ) \\
        & = &u_i^T ( u_{i+1} + b_{i+1} U_c - u_i - b_i U_c)\\
        & =& u_i^T  u_{i+1} - 1 + (b_{i+1}-b_i) u_i^T U_c
\end{eqnarray*}
where we used $\|u_i\|=1$.
But 
\begin{eqnarray*}
u_i^T  u_{i+1} - 1 & = &-\frac{1}{2}( u_{i+1}-u_i)^T ( u_{i+1}-u_i) \\
         & = &  -\frac{1}{2} \| u_{i+1}-u_i \|^2
\end{eqnarray*}
Thus
\begin{equation}\label{dot_di_Uc}
 \dot{d}_i = -\frac{1}{2} \| u_{i+1}-u_i \|^2 + (b_{i+1}-b_i) u_i^T U_c
\end{equation}
Let $V_i^c =(b_{i+1}-b_i)  U_c$
\begin{equation}\label{dot_di_Uc-2}
  \dot{d}_i \leq -\frac{1}{2} \| u_{i+1}-u_i \|^2 + \|V_i^c\| = -\frac{1}{2} \| u_{i+1}-u_i \|^2 +|b_{i+1}-b_i|\|U_c\|
\end{equation}

\begin{eqnarray}
  \sum_{i=1}^n \dot{d}_i  &\leq & -\frac{1}{2} \sum_{i=1}^n \| u_{i+1}-u_i \|^2 + \|U_c\| \sum_{i=1}^n |b_{i+1}-b_i| \\
    & \leq & -\frac{1}{2n} (\sum_{i=1}^n \| u_{i+1}-u_i \|)^2 +n_b \|U_c\| \label{sum_dot_di_Uc}
\end{eqnarray}
where $n_b$ is the number of pairs of agents $\{i,i+1\}$ such that $ |b_{i+1}-b_i| =1$ and we used the Cauchy-Schwartz inequality, (\ref{C-S}) in Appendix \ref{proof_Tm}. Since $n_b \leq n$ we  can write
\begin{equation}\label{sum_dot_di_Uc-2}
  \sum_{i=1}^n \dot{d}_i  \leq  -\frac{1}{2n} (\sum_{i=1}^n \| u_{i+1}-u_i \|)^2 +n \|U_c\|
\end{equation}
Repeating the reasoning in \cite{Richardson} and Appendix \ref{proof_Tm} we can show that 
$\sum_{i=1}^n \| u_{i+1}-u_i \| >1$ and therefore 
\begin{equation*}
  \sum_{i=1}^n \dot{d}_i  \leq  -\frac{1}{2n}  +n \|U_c\|
\end{equation*}

Thus, if  $\displaystyle \|U_c\| < \frac{1}{2 n^2} $ then  $ \sum_{i=1}^n \dot{d}_i < 0$

(b): After any capture, say $i$ captures $i+1$, the two agents  merge and thus $n$ is reduced. Let $\hat{n}$ be the number of remaining agents at time $t$ and $\hat{d}_i; i=1,...., \hat{n}$ the new distances. Then, by the same reasoning as for the proof of Lemma \ref{sum_di_dot-1} we have
\begin{equation}\label{sum_dot_di_hat}
  \sum_{i=1}^{\hat{n}} \dot{\hat{d}}_i  \leq  -\frac{1}{2 \hat{n}}  + \hat{n} \|U_c\|
\end{equation}
But, by definition,  $\sum_{i=1}^n d_i = \sum_{i=1}^{\hat{n}} \hat{d}_i$ and $\hat{n} \leq n  $ , thus
 \begin{eqnarray}
   \sum_{i=1}^n \dot{d}_i & = & \sum_{i=1}^{\hat{n}} \dot{\hat{d}}_i \\
    & \leq & -\frac{1}{2 \hat{n}}  + \hat{n} \|U_c\| \\
    & \leq & -\frac{1}{2 n}  + n \|U_c\|\label{sum_dot_di_n}
 \end{eqnarray}

(c): The distances $d_i(t)$ are continuous and there can be only a finite number of captures, thus by integrating (\ref{sum_dot_di_n}) we obtain
\begin{equation*}
  \sum d_i(T_m) = 0 \leq \sum d_i(t_0) + (-\frac{1}{2 n} +n \|U_c\|)(T_m-t_0)
\end{equation*}
\begin{equation}\label{TerminationTime}
  T_m \leq t_0 +\frac{2 n \sum d_i(t_0)}{1- 2 n^2 \|U_c\|}
\end{equation}
\end{proof}

\begin{remark}\label{stringentUc}
The condition $\|U_c\|<\frac{1}{2 n^2}$ is a very stringent bound for gathering and moving as a single point, since in general $n_b < n$ and we do not require mutual capture, i.e $n$ can decrease with time. 

\end{remark}

\begin{remark}
The bound (\ref{TerminationTime}) for gathering time  holds only for  $\|U_c\|<\frac{1}{2 n^2}$.

\end{remark}

 \section{Illustration by simulation of emergent behaviour in bearing-only cyclic pursuit}\label{SimComp}
 We simulated the cyclic pursuit with broadcast control to test the theory developed above.
 \subsection{Simulation parameters}
 \begin{itemize}
 \item User defined number of agents, $n$, and external velocity control, $U_c$. Since in our "bugs" model the speed of agents is one we limited  $U_c$, to  $\|U_c\| \leq 1$. 
 \item Randomly selected initial positions and agents detecting the external control
 \item Bearing-only simulation description
  \begin{itemize}
     \item $i$ captures $i+1$ (and merges with it) at time $t_c$, if $d_i(t_c) \leq \epsilon$  or   $i$ overtakes $i+1$. In the presented simulations we used $\epsilon=0.001$
      \item  For $t \geq t_c$ 
     
     \begin{itemize}
      \item $p_i(t) = p_{i+1}(t )$
       \item $\hat{b}_{i+1} = \hat{b}_i \bigvee b_{i+1}$
       \item $\displaystyle \dot{p}_{i+1}(t ) = \frac{p_{i+2}(t)-p_{i+1}(t)}{d_{i+1}}+\hat{b}_{i+1} U_c$
     \end{itemize}
       
  \end{itemize}
 \end{itemize}
 
 \subsection{Examples of simulations results}
 In this section we show sample simulation results for various 
   initial topologies,  various external inputs and sets of ad-hoc leaders. All the presented simulation results used $n=6$.
 
\subsubsection{Example1-bugs: Gathering property of the homogeneous "bugs" model}
We show an example of gathering  of the "bugs" model, without external input
 Let Example1-bugs denote the case of initial topology as in Fig. \ref{fig-Ex1Topo} and $U_c=(0,0)$.
\begin{figure}[H]
\begin{center}
\includegraphics[scale=0.6]{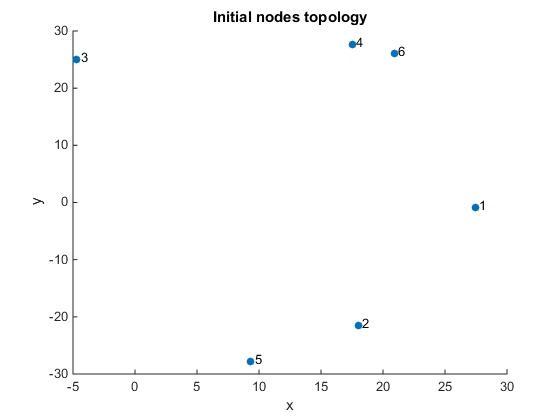}
\caption{Initial agents topology for Example1-bugs }\label{fig-Ex1Topo}
\end{center}
\end{figure}

\begin{figure}[H]
\begin{center}
\includegraphics[scale=0.6]{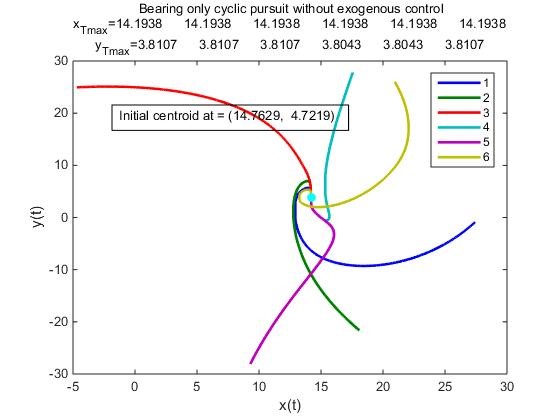}
\caption{Example1-bugs: Agents gathering in case of bearing-only cyclic pursuit without external input }\label{fig-Ex1-BearTraj-u0}
\end{center}
\end{figure}

In Fig. \ref{fig-Ex1-BearTraj-u0},   $(x_{Tmax},y_{Tmax})$ denote the position of each agent, at the end of the time interval. We see that the agents indeed converged to a point, but this differs from the (displayed) initial centroid.

For comparison, we show the behaviour of the agents, starting from the same initial conditions, but performing linear cyclic pursuit. The gathering point in this case is the initial centroid.

\begin{figure}[H]
\begin{center}
\includegraphics[scale=0.6]{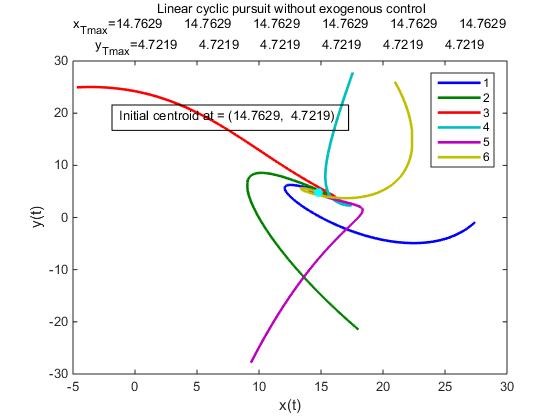}
\caption{ Agents gathering in case of linear cyclic pursuit without external input }\label{fig-Ex1-LinTraj-u0}
\end{center}
\end{figure}

\subsubsection{Example2-bugs: Impact of external input and incomplete sets of ad-hoc leaders} 
Next we show the impact of an external input, $U_c$, on the emergent behaviour, for various values of $U_c$, such that $\|U_c\| \leq 1$, and various incomplete sets of agents detecting it. We show for each example the behaviour of $d_i(t)$, the distance  of  agent $i$ to $i+1$, $\forall i$, as well as agents' trajectories and  velocities.
 We observe that in all considered cases
\begin{itemize}
\item There exists a time $t_c$ of mutual capture where $d_i=0 \quad \forall i$
\item For $t \geq t_c$ all agents move as a single point with velocity $U_c$
\item The value of $t_c$ and the behaviour of distances to prey, $d_i(t)$ for $t < t_c$, as well as of agents' velocities for $t < t_c$ depend on the value of the external input and on the set of agents detecting the broadcast signal
\end{itemize} 

All cases of Example2-bugs were run starting from the positions shown in Fig. \ref{fig-Ex2Topo}.  
\begin{figure}[H]
\begin{center}
\includegraphics[scale=0.6]{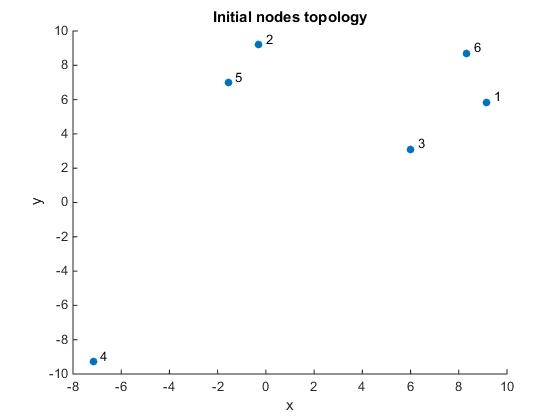}
\caption{Initial agents positions for Example2-bugs }\label{fig-Ex2Topo}
\end{center}
\end{figure}

Cases considered were
\begin{itemize}
\item \textbf{Example2.1-bugs}: Same broadcast signal, different ad-hoc leaders
\begin{itemize}

\item Example2.1.1-bugs: $U_c=(0.5,0.3)^T$, $B=(1 1 1 0 1 1 )^T $
\item Example2.1.2-bugs: $U_c=(0.5,0.3)^T$, $B=(0 0 1 0 0 0)^T$
\end{itemize}
\item  \textbf{Example2.2-bugs}: Same set of ad-hoc leaders, increasing magnitude of broadcast signal
\begin{itemize}

\item Example2.2.1-bugs: $B=(0 1 0 0 1 0 )^T $,  $U_c=(0.013, 0)^T$
\item Example2.2.2-bugs: $B=(0 1 0 0 1 0 )^T $,  $U_c=(0.5,0.3)^T$
\item Example2.2.3-bugs: $B=(0 1 0 0 1 0 )^T $,  $U_c=(-1,0)^T$

\end{itemize}

\end{itemize}
where $B$ is a vector of pointers to the agents detecting the exogenous control.

\paragraph{Example2.1-bugs: Same broadcast signal, different ad-hoc leaders\\}

\underline{\textbf{Chaser to prey distances - $d_i(t)$}}

\begin{figure}[H]
\begin{center}
\includegraphics[scale=0.5]{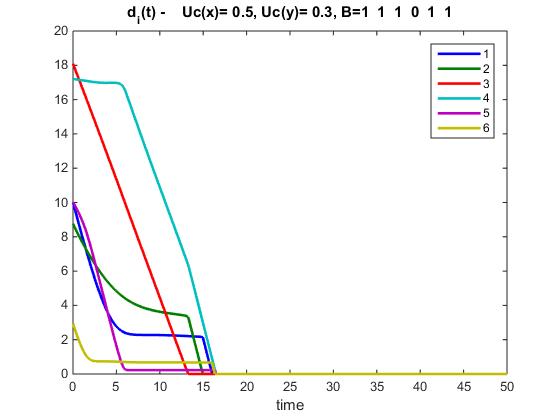}
\caption{ Example2.1.1-bugs: Distances behaviour }\label{fig-Di-u0503B111011}
\end{center}
\end{figure}

\begin{figure}[H]
\begin{center}
\includegraphics[scale=0.5]{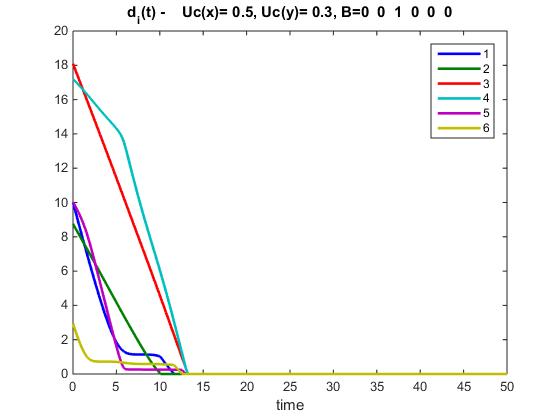}
\caption{Example2.1.2-bugs: Distances behaviour  }\label{fig-Di-u0503B3}
\end{center}
\end{figure}

\newpage
\underline{\textbf{Velocities of agents}}

\begin{figure}[H]
\begin{center}
\includegraphics[scale=0.5]{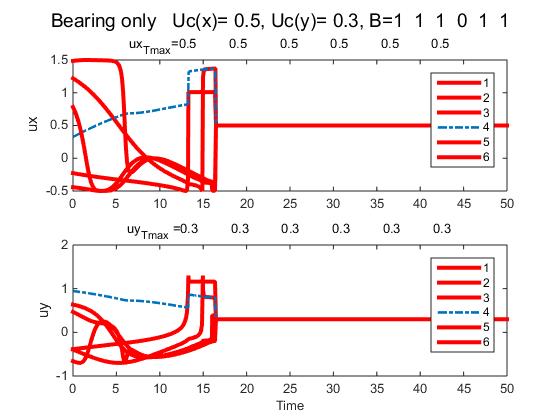}
\caption{ Example2.1.1-bugs: Agents' velocities  }\label{fig-Vel-u0503B111011}
\end{center}
\end{figure}

\begin{figure}[H]
\begin{center}
\includegraphics[scale=0.5]{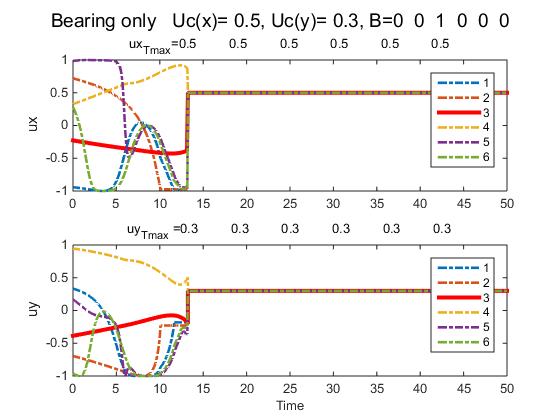}
\caption{ Example2.1.2-bugs: Agents' velocities }\label{fig-Vel-u0503B3}
\end{center}
\end{figure}

\newpage
\underline{\textbf{ Agents' Trajectories }}

\begin{figure}[H]
\begin{center}
\includegraphics[scale=0.5]{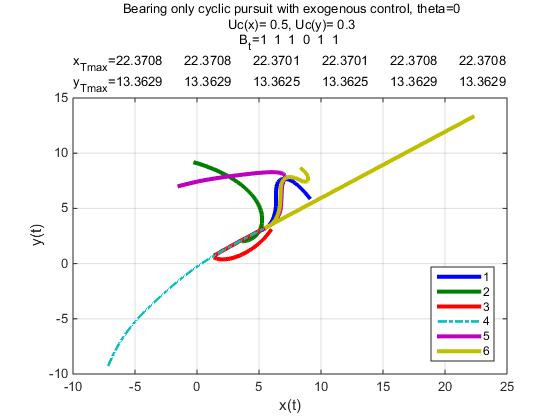}
\caption{ Example2.1.1-bugs: Agents' trajectories  }\label{fig-Traj-u0503B111011}
\end{center}
\end{figure}

\begin{figure}[H]
\begin{center}
\includegraphics[scale=0.5]{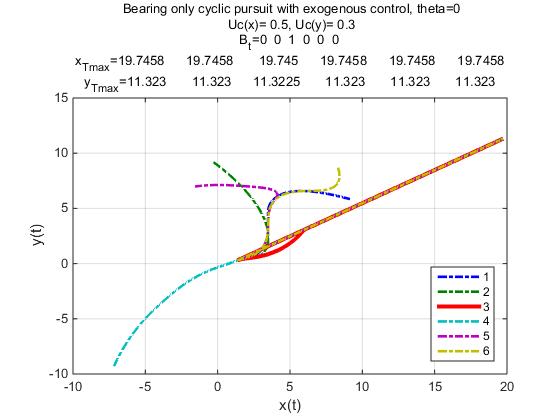}
\caption{Example2.1.2-bugs: Agents' trajectories  }\label{fig-Traj-u0503B3}
\end{center}
\end{figure}

\newpage
\paragraph{Example2.2-bugs: Same set of ad-hoc leaders, increasing magnitude of broadcast signals\\}

\underline{\textbf{Chaser to prey distances -$d_i(t)$}}

\begin{figure}[H]
\begin{center}
\includegraphics[scale=0.5]{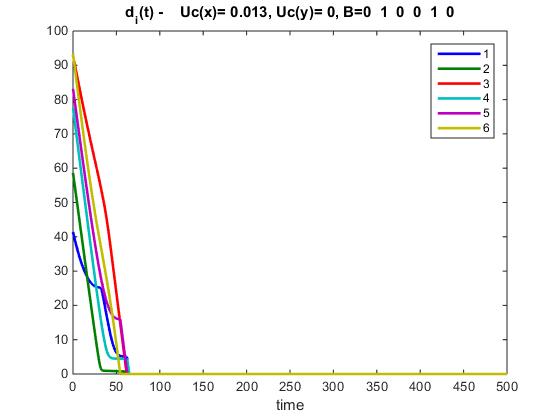}
\caption{ Example2.2.1-bugs: Distances behaviour  }\label{fig-Di-u00130B010010}
\end{center}
\end{figure}

\begin{figure}[H]
\begin{center}
\includegraphics[scale=0.5]{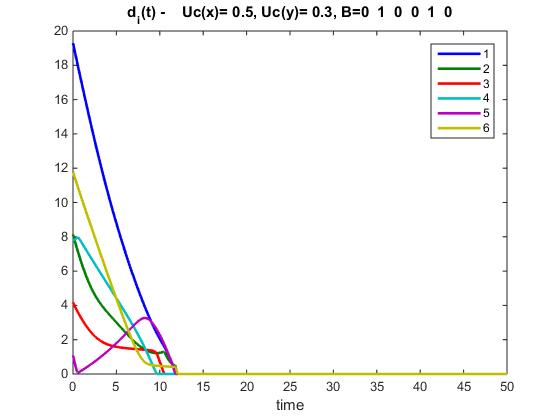}
\caption{ Example2.2.2-bugs: Distances behaviour  }\label{fig-Di-u0503B010010}
\end{center}
\end{figure}

\begin{figure}[H]
\begin{center}
\includegraphics[scale=0.5]{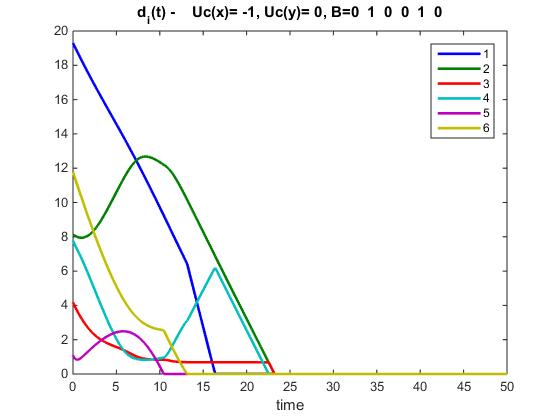}
\caption{Example2.2.3-bugs: Distances behaviour  }\label{fig-Di-u-10B010010}
\end{center}
\end{figure}

\newpage
\underline{\textbf{Agents' velocities}}

\begin{figure}[H]
\begin{center}
\includegraphics[scale=0.5]{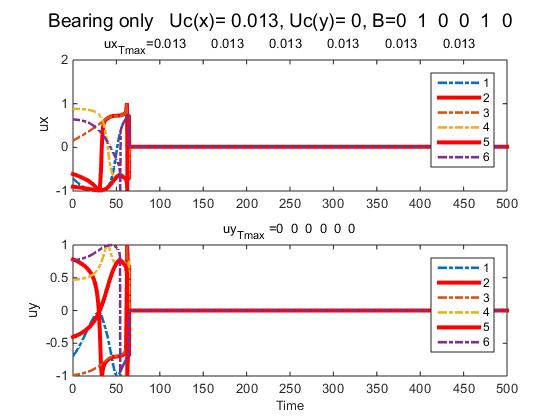}
\caption{Example2.2.1-bugs: Velocities behaviour }\label{fig-Vel-u00130B010010}
\end{center}
\end{figure}

\begin{figure}[H]
\begin{center}
\includegraphics[scale=0.5]{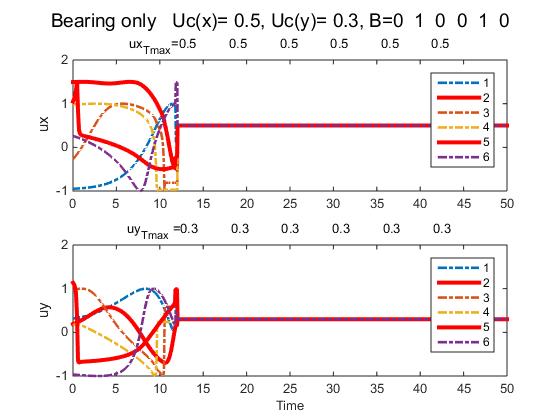}
\caption{Example2.2.2-bugs: Velocities behaviour  }\label{fig-Vel-u0503B010010}
\end{center}
\end{figure}

\begin{figure}[H]
\begin{center}
\includegraphics[scale=0.5]{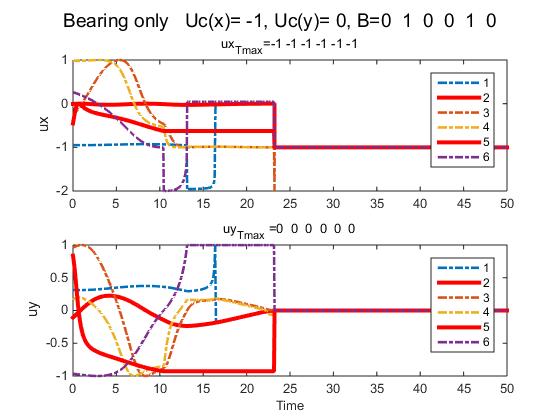}
\caption{ Example2.2.3-bugs: Velocities behaviour  }\label{fig-Vel-u-10B010010}
\end{center}
\end{figure}

\newpage
\underline{\textbf{Agents' trajectories}}

\begin{figure}[H]
\begin{center}
\includegraphics[scale=0.5]{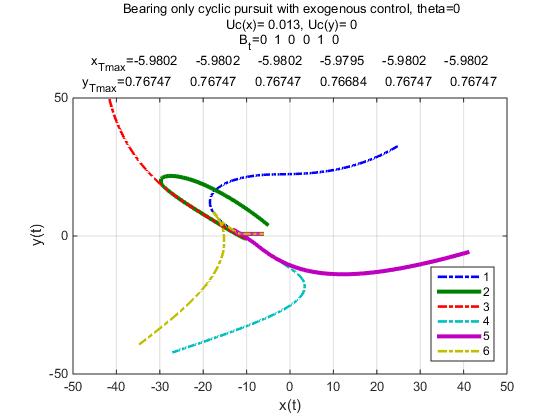}
\caption{Example2.2.1-bugs:  Trajectories behaviour  }\label{fig-Traj-u00130B010010}
\end{center}
\end{figure}

\begin{figure}[H]
\begin{center}
\includegraphics[scale=0.5]{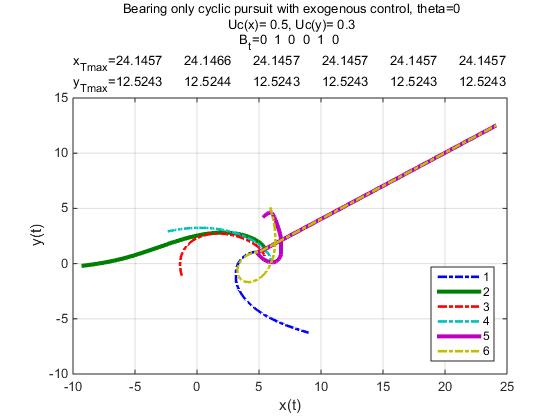}
\caption{Example2.2.2-bugs: Trajectories behaviour   }\label{fig-Traj-u0503B010010}
\end{center}
\end{figure}

\begin{figure}[H]
\begin{center}
\includegraphics[scale=0.5]{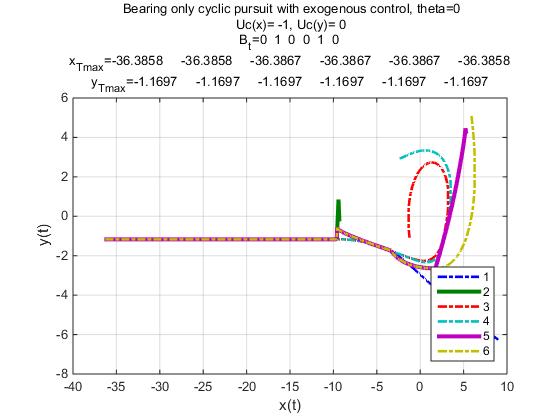}
\caption{Example2.2.3-bugs: Trajectories behaviour   }\label{fig-Traj-u-10B010010}
\end{center}
\end{figure}

\newpage
\subsubsection{Example3-bugs: Broadcast signal received by all}
In this section we show that if the broadcast control is received by all then the gathering property of the cyclic pursuit is  
 independent of the value of $ \|Uc\|$. We show simulation results for 
 \begin{itemize}
 \item Example3.1-bugs:$U_c=(5,3)^T$
 \item Example3.2-bugs:$U_c=(-3,2)^T$
 \item Example3.3-bugs: $U_c=(0,0)^T$
 \end{itemize}
From the presented results we observe that the time to convergence is identical in all cases, i.e. identical to the case $U_c=(0,0)^T$. 
These results correspond to the theory in section \ref{SingleDist}, case 2, where we show that if $b_i=b_{i+1} =1$ then $\dot{d}_i$ is independent of $U_c$.

\underline{\textbf{Distances behaviour -$d_i(t)$}}
\begin{figure}[H]
\begin{center}
\includegraphics[scale=0.5]{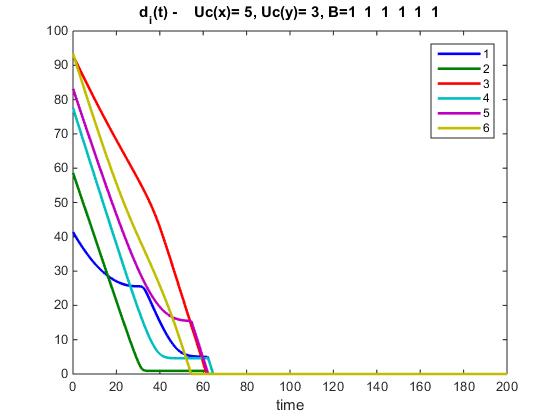}
\caption{Example3.1-bugs: Distances behaviour  }\label{fig-Ex3-Ball-U53-Di}
\end{center}
\end{figure}

\begin{figure}[H]
\begin{center}
\includegraphics[scale=0.5]{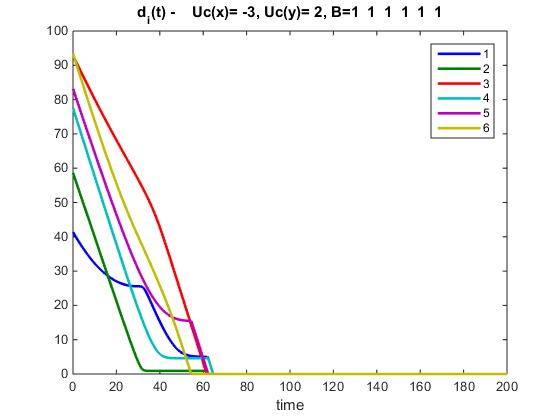}
\caption{Example3.2-bugs: Distances behaviour  }\label{fig-Ex3-Ball-U-32-Di}
\end{center}
\end{figure} 

\begin{figure}[H]
\begin{center}
\includegraphics[scale=0.5]{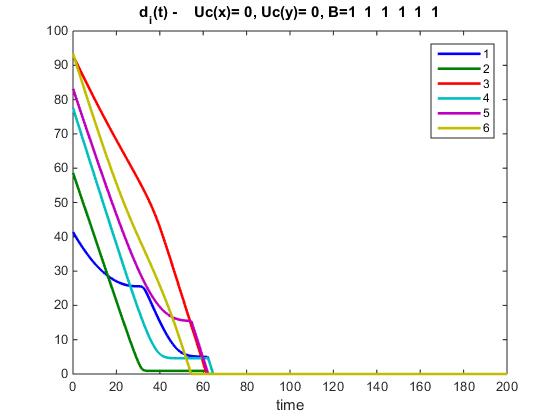}
\caption{Example3.3-bugs: Distances behaviour in "bugs" cyclic pursuit, without external control}\label{fig-Ex3-Ball-U00-Di}
\end{center}
\end{figure}

 In all the examples shown above, we obtained
 \begin{itemize}
 \item Identical distances behaviour, i.e. independent of $U_c$
 \item The time to convergence was 64.5979
 \end{itemize}

.

\newpage 

\underline{\textbf{Agents' velocities}}

\begin{figure}[H]
\begin{center}
\includegraphics[scale=0.5]{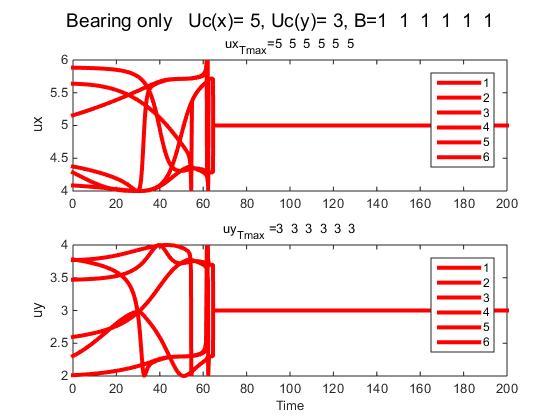}
\caption{Example3.1-bugs:  Velocities behaviour}\label{fig-Ex3-Ball-U53-Vel}
\end{center}
\end{figure}

\begin{figure}[H]
\begin{center}
\includegraphics[scale=0.5]{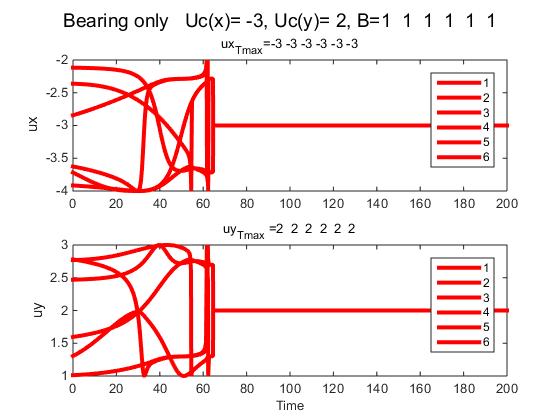}
\caption{Example3.2-bugs: Velocities behaviour }\label{fig-Ex3-Ball-U-32-Vel}
\end{center}
\end{figure}

\newpage

\underline{\textbf{Agents' trajectories}}

\begin{figure}[H]
\begin{center}
\includegraphics[scale=0.5]{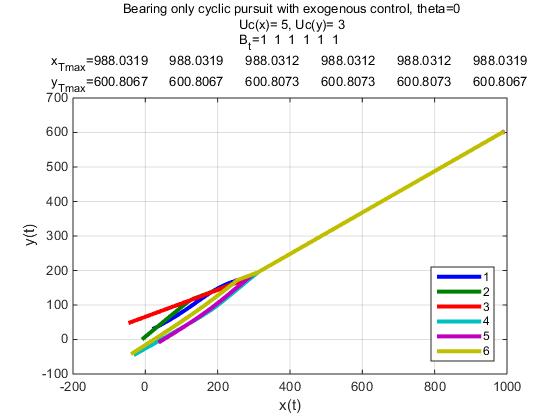}
\caption{Example3.1-bugs: Trajectories behaviour }\label{fig-Ex3-Ball-U53-Traj}
\end{center}
\end{figure}

\begin{figure}[H]
\begin{center}
\includegraphics[scale=0.5]{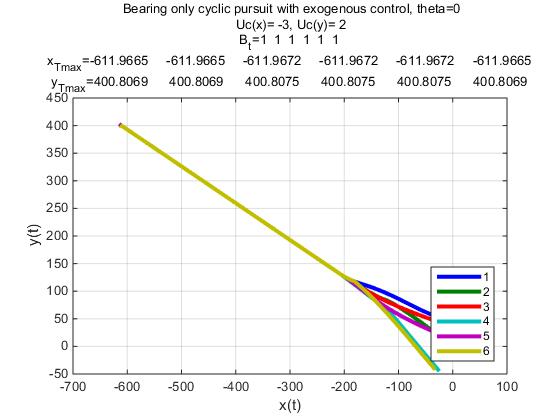}
\caption{ Example3.2-bugs: Trajectories behaviour }\label{fig-Ex3-Ball-U-32-Traj}
\end{center}
\end{figure}

\begin{figure}[H]
\begin{center}
\includegraphics[scale=0.5]{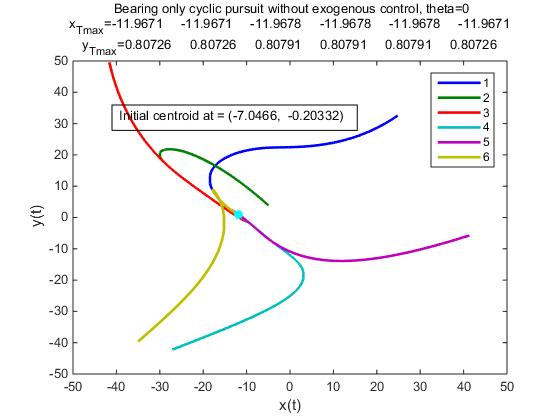}
\caption{Example3.3-bugs:  Trajectories behaviour in "bugs" cyclic pursuit, when $U_c=(0,0)^T$}\label{fig-Ex3-Ball-U00-Traj}
\end{center}
\end{figure}

\begin{remark} 
The broadcast signal being received by all is a particular case of the broadcast signal being received by a random set of agents. Since we do not enforce it, in order to enable the general case of a random, not complete, set of  agents receiving the broadcast signal we need to enforce $\|Uc\| \leq 1$  
\end{remark}

\newpage
\appendix
\addcontentsline{toc}{section}{Appendices}
\section{About matrices}\label{matrices}
Following \cite{HJbook}, let $\mathbf{M}_n$ denote the class of all $n \times n$ matrices. Two matrices $A, B \in M_n$ are similar, denoted by $A \sim B$, if there exists an invertible (non-singular) matrix $S \in M_n$ s.t. $A = SBS^{-1}$. Similar matrices are just different basis representation of a single linear transformation.Similar matrices have the same characteristic polynomial, c.f. Theorem 1.3.3 in \cite{HJbook} and therefore the same eigenvalues

\subsection{Algebraic and geometric multiplicity of eigenvalues}\label{EigMultiplicity}

Let $\lambda$ be an eigenvalue of an arbitrary matrix $A \in M_n$ with an associated eigenvector $v \in \mathbb{C}^n$.\\
\textbf{Definitions:}
\begin{itemize}
  \item The spectrum of $A \in M_n$  is the set of all the eigenvalues of $A$, denoted by $\sigma ( A )$.
  \item The spectral radius of $A$ is $\rho ( A ) = max {| \lambda | : \lambda \in \sigma ( A ) }$.
  \item  For a given $\lambda \in \sigma(A)$ , the set of all vectors $v \in \mathbb{C}^n$ satisfying $A v = \lambda v$ is called the eigenspace of $A$ associated with the eigenvalue $\lambda$ . Every nonzero element of this eigenspace is an eigenvector of $A$ associated with $\lambda$
  \item The \emph{algebraic multiplicity} of $\lambda$ is its multiplicity as a root of the characteristic polynomial $det(\lambda I - A)$
  \item The \emph{geometric multiplicity} of  $\lambda$ is the dimension of the eigenspace associated with $\lambda$, i.e. the number of linearly independent eigenvectors associated with that eigenvalue.
  \item  We say that $\lambda$ is simple if its algebraic multiplicity is 1; it is semisimple if its algebraic and geometric multiplicities are equal.
  \item The algebraic multiplicity of an eigenvalue is larger than or equal to its geometric multiplicity.
  \item  We say that $A$ is defective if the geometric multiplicity of some eigenvalue is less than its algebraic multiplicity
\end{itemize}

\subsection{Diagonizable matrices}\label{DiagMat}

\textbf{Definition}: If $A \in M_n$ is similar to a diagonal matrix, then $A$ is said to be diagonizable

\begin{theorem}\label{th-diagonizable}
See Theorem 1.3.7 in \cite{HJbook}.

The matrix $A \in M_n$ is diagonizable iff there are $n$ linearly independent vectors, $v^{(1)},v^{(2)}, ....,v^{(n)}$,  each of which is an eigenvector of $A$. If $v^{(1)},v^{(2)}, ....,v^{(n)}$ are linearly independent eigenvectors of $A$ and $S=[v^{(1)},v^{(2)}, ....,v^{(n)} ]$ then $S^{-1} A S$ is a diagonal matrix $\Lambda$ and the diagonal entries of $\Lambda$ are the eigenvalues of $A$
\end{theorem}

\begin{lemma}\label{distinct-k}
Let $\lambda_1,...,\lambda_k; \quad k \geq 2$ be distinct eigenvalues of $A \in M_n$ and suppose $v^{(i)}$ is an eigenvector associated with $\lambda_i; \quad i=1,....,n$. Then the vectors $[v^{(1)},v^{(2)}, ....,v^{(k)} ]$ are linearly independent.
\end{lemma}
Proof of Lemma 1.3.8 in \cite{HJbook}

\begin{theorem}\label{distinct-all}
If $A \in M_n$ has $n$ distinct eigenvalues, then $A$ is diagonizable
\end{theorem}
\begin{proof}
Since all eigenvalues are distinct Lemma \ref{distinct-k} ensures that the associated eigenvectors are linearly independent and thus, according to Theorem \ref{th-diagonizable}, $A$ is diagonizable
\end{proof}
\textbf{Notes:}
\begin{enumerate}
  \item Having distinct eigenvalues is sufficient for diagonizability, but not necessary.
  \item A matrix is diagonizable iff it is non-defective, i.e. it has no eigenvalue with geometric multiplicity strictly less than its algebraic multiplicity
\end{enumerate}

\subsection{Left eigenvectors}\label{LeftEigvec}
\textbf{Definition}: A non-zero vector $y \in \mathcal{C}^n$ is a left eigenvector of $A \in M_n$ associated with eigenvalue $\lambda$ of $A$ if $y^*A = \lambda y^*$.
From \cite{HJbook}, Theorem 1.4.12, we have the following relationship between left and right eigenvectors and the multiplicities of the corresponding eigenvalue:
\begin{theorem}\label{left-right}
Let $\lambda \in \mathcal{C}$ be an eigenvalue of $A \in M_n$ associated with right eigenvector $x$ and left eigenvector $y^*$. Then the following hold:
\begin{enumerate}[label=(\alph*)]
\item If $\lambda$ has algebraic multiplicity 1, then $y^*x \neq 0$
\item If $\lambda$ has geometric multiplicity 1, then it has algebraic multiplicity 1 iff $y^*x \neq 0$
\end{enumerate}
\end{theorem}

\subsection{About non-symmetric real matrices}\label{non-sym}
\begin{itemize}
  \item the eigenvalues of non-symmetric real $n\times n$ matrix are real or come in complex conjugate pairs
  \item the eigenvectors are not orthonormal in general and may not even span an n-dimensional space
      \begin{itemize}
        \item Incomplete eigenvectors can occur only when there are degenerate eigenvalues, i.e. eigenvalues with algebraic multiplicity greater than 1, but do not always occur in such cases
        \item Incomplete eigenvectors never occur for the class of normal matrices
      \end{itemize}
  \item Diagonalization theorem: an $n\times n$ matrix $A$ is diagonizable iff $A$ has $n$ linearly independent eigenvectors
\end{itemize}

\subsection{Normal matrices}\label{Normal}
\begin{definition}\label{NormalMat}
A matrix $A \in M_n$ is called normal if $A^* A = A A^*$
\end{definition}
\begin{definition}\label{Hermitian}
A matrix $A \in M_n$ is called Hermitian if $A^* =  A$
\end{definition}
\begin{theorem}\label{SpectralNormal}
If $A \in M_n$ has eigenvalues $\lambda_1, \lambda_2,.....,\lambda_n$ the following statements are equivalent:
\begin{enumerate}[label=(\alph*)]
\item $A$ is normal
\item $A$ is unitarily diagonizable
\item $\sum_{i=1}^n \sum_{j=1}^n |a_{ij}|^2 = \sum_{j=1}^n |\lambda_j|^2$
\item There is an orthonormal set of $n$ eigenvectors of $A$
\end{enumerate}
\end{theorem}

\begin{remark}
 All normal matrices are diagonizable but not all diagonizable matrices are normal.
\end{remark}

\subsection{Unitary matrices and unitary similarity}\label{UniMat}
Unitary matrices, $U \in M_n$, are non-singular matrices such that  $U^{-1} = U^*$, i.e $U^*U =U U^* =I$. A real matrix $U \in M_n(\mathbb{R})$ is real orthogonal if $U^T U=I$.
The following are equivalent:
\begin{enumerate}[label=(\alph*)]
\item $U$ is unitary
\item $U$ is non-singular and $U^{-1} = U^*$
\item $U U^* =I$
\item $U^*$ is unitary
\item The columns of $U$ are orthonormal
\item The rows of $U$ are orthonormal
\item For all $x \in \mathcal{C}^n, \quad \|x\|_2=\|Ux\|_2$
\end{enumerate}
\begin{defn}:
\begin{itemize}
  \item $A$ is unitarily similar to $B$ if there is a unitary matrix $U$ s.t. $A=U B U^*$
  \item $A$ is unitarily diagonizable if it is unitarily similar to a diagonal matrix
\end{itemize}
\end{defn}

\section{Laplacian representation of Graphs }\label{preliminary}
Graphs provide natural abstraction for how information is shared between agents in a network. Algebraic graph theory associate matrices, such as Adjacency and Laplacian, with graphs, cf. \cite{Mesbahi-book}.
In this appendix some useful facts from algebraic graph theory are presented. Given a multi-agent system, the network can be represented by a directed or an undirected graph  $G=(V,E)$, where  $V$ is a finite set of vertices, labeled by $i; \quad i=1,..., n$ representing the agents, and $E$ is the set of edges, $E \in [V \times V]$, representing inter-agent information exchange links.\footnote{
Vertices are also referred to as nodes and the two terms will be used interchangeably}
 A simple graph contains no self-loops, namely there is no edge from a node to itself. If the graph is undirected then the edge set $E$ contains unordered pairs of vertices.
 In directed graphs (digraphs) the edges are ordered pairs of vertices.
We say that the graph is connected if for every pair of vertices in $V$  there is a path with those vertices as its end vertices.
If this is not the case, the graph is called disconnected.
We refer to a connected graph as having one connected component. A disconnected graph has more than one component.

\subsection{Directed graphs - digraphs}\label{Digraph}
A directed graph (or digraph), denoted by $D= (V, E)$, is a graph whose edges are ordered pairs of vertices. For the ordered pair $(i, j) \in E$, when vertices $v_i, v_j$ are labelled $i, j$,  $i$ is said to be the tail  of the edge, while $j$ is its head.

\textbf{Definitions:}\\
 \begin{enumerate}
  \item A digraph is called strongly connected if for every pair of vertices there is a \emph{directed path} between them.
   \item The digraph is called weakly connected if it is connected when viewed as a graph, that is, a disoriented digraph.
   \item A digraph has a rooted out-branching, or spanning tree, if there exists a vertex $r$ (the root) such that for every other vertex $i \neq r \in N$ there is a directed path from $r$ to $i$. In this case, every  $i \neq r \in N$ is said to be reachable from $r$. In strongly connected digraphs each node is a root.
   \item A node is called balanced if  the total weight of edges entering the node and leaving the same node are equal
    \item If all nodes in the digraph are balanced then the digraph is called balanced

\end{enumerate}

\subsubsection{Properties of Laplacian matrices associated with digraphs}\label{L_D}
\begin{itemize}
  \item The non-symmetric Laplacian, $L$, associated with a  digraph $G$ of order $n$  has the following properties:
      \begin{enumerate}[label=(\alph*)]
          \item\label{itm:first}$L$  has at least one zero eigenvalue and all remaining eigenvalues have positive real part
          \item\label{itm:second}  $L$ has a simple zero eigenvalue and all other eigenvalues have positive real part if and only if $G$ has a directed spanning tree
          \item\label{itm:third} $L$ is real, therefore any complex eigenvalues must occur in conjugate pairs\footnote{ the eigenvalues of a real non-symmetric matrix may include real values, but may also include pairs of complex conjugate values}
          \item\label{itm:fourth} There is a right eigenvector of ones, $\mathbf{1}_n$, associated with the zero eigenvalue, i.e. $L \mathbf{1}_n = \mathbf{0}_n$
          \item\label{itm:fifth}  The left eigenvector of $L$ corresponding to $\lambda=0$, denoted by $w_l$ is positive and $\sum_{i=1}^n w_l(i) =1$,
           \item\label{itm:sixth} $w_l = \mathbf{1}_n^T$ \emph{if and only if the digraph is balanced}

      \end{enumerate}

  \item If the Laplacian $L$ of the digraph is a normal, i.e. $L L^T = L^T L$, then
  \begin{enumerate}
    \item There exists an orthonormal set of $n$ eigenvectors of $L$
    \item $L$ is unitarily diagonizable, i.e. $L = U \Lambda U^*$, where $U$ is a unitary matrix of eigenvectors and $\Lambda$ is a diagonal matrix of eigenvalues.
    \item The digraph must be balanced and thus $w_l = \mathbf{1}_n^T$
  \end{enumerate}
\end{itemize}

\section{Properties of circulant matrices}\label{circulant}
A circulant matrix is an $ n \times n$ matrix having the form
\begin{equation}\label{eq-circ}
   C  =
          \begin{bmatrix}
           c_0 & c_1 & c_2  & \hdots & c_{n-1} \\
             c_{n-1} & c_0 & c_1 & \hdots & c_{n-2} \\
	   &&\vdots \\
            c_1 & c_2  && \hdots  & c_0 \\
          \end{bmatrix}
  \end{equation}

which can also be characterized as an $ n \times n$  matrix $C$ with entry $(k,j)$ given by
\begin{equation*}
  C_{k,l}=c_{(l-k) mod (n)}
\end{equation*}

Every $ n \times n$  circulant matrix $C$ has eigenvectors  (cf. \cite{Gray}, \cite{RamirezPhD})
\begin{equation}\label{eig-vec}
  v_k=\frac{1}{\sqrt{n}} \left ( 1, e^{-2 \pi jk/n}, e^{-4 \pi jk/n},\hdots, e^{-2 \pi jk(n-1)/n}  \right)^T; \quad k \in \left \{0, 1, \hdots, n-1 \right \}
\end{equation}
where $ j=\sqrt{-1}$, with corresponding eigenvalues
\begin{equation}\label{eig}
  \lambda_k = \sum_{l=0}^{n-1} c_l e^{-2 \pi j lk/n}
\end{equation}

From the definition of eigenvalues and eigenvectors we have
\begin{equation*}
  C v_k = \lambda_k v_k; \quad k= 0,1, \hdots, n-1
\end{equation*}
which can be written as a single matrix equation
\begin{equation*}
  C U= U \Lambda
\end{equation*}
where $\Lambda = diag(\lambda_k); \quad k=0, \hdots, n-1$ and
\begin{eqnarray*}
  U &=& \left [v_0, v_1, \hdots, v_{n-1} \right ] \\
   &=& \frac{1}{\sqrt{n}}\left [e^{-2\pi j mk}; \quad m,k = 0, \hdots, n-1 \right ]
\end{eqnarray*}
$U$ is a unitary matrix, i.e. $ UU^* = U^*U=I$ (cf. \cite{Gray}, proof by direct computation) and
\begin{equation}\label{C-diagonalization}
  C=U \Lambda U^*
\end{equation}
Note that $F_n = \sqrt{n} U^*$ is the known Fourier matrix.

\subsection{Cyclic pursuit}\label{cyclic}
The Laplacian representing cyclic pursuit is a special case of circulant matrix
\begin{equation}\label{eq-cyclic}
   L  =
          \begin{bmatrix}
           1 & -1 & 0  & 0 & \hdots & 0 \\
             0 & 1 & -1 & 0 & \hdots & 0 \\
	   \vdots && &&& \vdots \\
            -1 & 0  && & \hdots  & 1 \\
          \end{bmatrix}
  \end{equation}
Thus the eigenvalues of the cyclic pursuit Laplacian are
\begin{equation}\label{eig-cyclic}
  \lambda_k = 1-  e^{-2 \pi j k/n}; \quad k=0, ..., n-1
\end{equation}
and the eigenvectors are given by eq. (\ref{eig-vec}).

\section{Proof of mutual capture existence in finite time, in  non-linear cyclic pursuit without broadcast control - Lemma 1.1 in \cite{Richardson}}\label{proof_Tm}
This proof, sketched in \cite{Richardson}, is brought here for completeness.
Let $p_i(t)$ be the position of agent $i$ at time $t$ and let agent $i$ chase $i+1$, where $i$ is mod $n$. Denote by 
$d_i(t)$ is the distance between $i$ to $i+1$ at time $t$, i.e. 
 $d_i(t) = \|p_{i+1}(t)-p_i(t)\|$.
 
 The dynamics of the agents are modeled by
\begin{equation}\label{CyclicNonLinearControl}
 \begin{cases}
  \dot{p_i}=\frac{p_{i+1}(t)-p_i(t)}{d_i(t)} \quad \text{if} \quad d_i(t) >0\\
  p_i(t)=p_{i+1}(t) \quad \forall t\geq \hat{t} \quad \text{s.t.} \quad d_i(\hat{t})=0
  \end{cases} 
\end{equation}

In the sequel, for simplicity,  we shall omit specific mention of $t$, whenever self explanatory.

\begin{lem}
$n$ agents in cyclic pursuit, modelled by (\ref{CyclicNonLinearControl}), will collide (gather) within a finite time given by $T_m \leq t_0+ 2 n \sum_{i=1}^n d_i(t_0)$, where $d_i(t_0)$ are the initial distances between agents and $T_m$ is the time of mutual capture (termination time).
\end{lem}

\begin{proof}
We show that there exists a time $T_m$ such that $d_i(T_m) =0$ for all $i$ or, since $d_i > 0 $ for all $i$
\begin{equation}\label{sum_di_Tm}
  \sum_{i=1}^n d_i(T_m)=0
\end{equation}
To show (\ref{sum_di_Tm}) we will show that there exists a positive real number $c>0$ such that $\sum_{i=1}^n \dot{d}_i(t \leq T_m) \leq -c$.

We assume that no agents have collided at time $t<T_m$. Note that upon our model, when two agents collide they become one and $n$ is reduced, therefore this assumption holds.
Given eq. (\ref{CyclicNonLinearControl}) we have
\begin{equation*}
  \sum_{i=1}^n d_i \dot{p}_i = 0
\end{equation*}
Thus
\begin{equation}\label{sum0}
  \dot{p}_1^T \sum_{i=1}^n d_i \dot{p}_i = \sum_{i=1}^n d_i \dot{p}_1^T \dot{p}_i = 0
\end{equation}
In order for (\ref{sum0}) to hold, there must exist an agent $j$ such that $\dot{p}_1^T \dot{p}_j <0$.
\begin{equation}\label{alpha}
  \dot{p}_1^T \dot{p}_j = \|\dot{p}_1\|\|\dot{p}_j\| \cos(\alpha) <0
\end{equation}
where $\alpha$ is the angle between $ \dot{p}_1$ and $ \dot{p}_j$. 
For eq. (\ref{alpha}) to hold we must have $\pi/2 < \alpha < 3 \pi/2 $ and since $\|\dot{p}_i\|=1; i=1,....,n$, we have $\|\dot{p}_j - \dot{p}_1 \| >1 $

\begin{figure}[H]
\begin{center}
\includegraphics[scale=0.5]{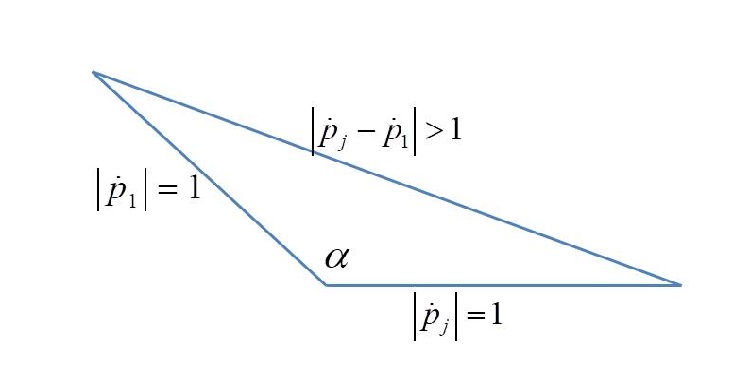}
\caption{$\|\dot{p}_j - \dot{p}_1 \|$ for $\dot{p}_1^T \dot{p}_j <0$ }
\end{center}
\end{figure}

From the definition of $d_i$ and $\dot{p}_i$ we have
\begin{eqnarray}
  d_i^2 &=& \|p_{i+1}-p_i \|^2  \\
  \dot{d}_i &=& \dot{p}_i^T (\dot{p}_{i+1}-\dot{p}_i ) \\
        & = &\dot{p}_i^T \dot{p}_{i+1} - 1\\
        & = & -\frac{1}{2}\|\dot{p}_{i+1}-\dot{p}_i \|^2
\end{eqnarray}

\begin{eqnarray*}
  \sum_{i=1}^n \dot{d}_i &= & -\frac{1}{2} \sum_{i=1}^n \|\dot{p}_{i+1}-\dot{p}_i \|^2\\
        &  \leq & -\frac{1}{2 n} (\sum_{i=1}^n \|\dot{p}_{i+1}-\dot{p}_i \|)^2
\end{eqnarray*}
where we used the Cauchy-Schwartz inequality
\begin{equation}\label{C-S}
  \sum_{i=1}^n a_i^2 \sum_{i=1}^n b_i^2 \geq (\sum_{i=1}^n a_i b_i)^2
\end{equation}
 with $a_i= \|\dot{p}_{i+1}-\dot{p}_i \|$ and $b_i=1$ for $i=1,....,n$.

Applying to $ \sum_{i=1}^n \|\dot{p}_{i+1}-\dot{p}_i \|$ the triangle inequality we have
\begin{equation}\label{sum_to_j}
  \sum_{i=1}^{j-1} \|\dot{p}_{i+1}-\dot{p}_i \| \geq \|\dot{p}_j-\dot{p}_1 \|
\end{equation}
where $j$ is some index such that $\|\dot{p}_j-\dot{p}_1 \| >1$, which according to (\ref{sum0}) must exist. Since all summands are $\geq 0$ we, have $ \sum_{i=1}^n \|\dot{p}_{i+1}-\dot{p}_i \| > 1$  and therefore
\begin{equation}\label{sum_d_dot}
  \sum_{i=1}^n \dot{d}_i < -\frac{1}{2 n}
\end{equation}
Integrating both sides of (\ref{sum_d_dot})  from $t_0$ to $T_m$ and recalling that at termination (mutual capture) time $\sum_{i=1}^n d_i(T_m)=0$,  we obtain
\begin{equation*}
  T_m < t_0+2n \sum_{i=1}^n d_i(t_0)
\end{equation*}
\end{proof}

\newpage
\bibliographystyle{abbrv}
\bibliography{BibGen}

\begin{thebibliography}{10}

\bibitem{BG1975}
F.~Behroozi and R.~Gagnon.
\newblock A computer-assisted study of pursuit in a plane.
\newblock {\em Amer. Math. Monthly}, 82(8):804–812, 1975.

\bibitem{BG1979}
F.~Behroozi and R.~Gagnon.
\newblock Cyclic pursuit in a plane.
\newblock {\em J.Math. Phys.}, 20:2212–2216, 1979.

\bibitem{Bruck1993}
A.~Bruckstein.
\newblock Why the ant trails look so straight and nice.
\newblock {\em The Mathematical Intelligencer}, 15(2):59--70, 1993.

\bibitem{Bruck}
A.~Bruckstein, N.~Cohen, and A.~Efrat.
\newblock Ants, crickets and frogs in cyclic pursuit.
\newblock Technical report, Technion, CS, Center for Intelligent Systems TR,
  CIS-9105, 1991.

\bibitem{DS2014}
S.~Daingade and A.~Sinha.
\newblock Target centric cyclic pursuit using bearing angle measurements only.
\newblock IFAC Proceedings Volumes, 2014.

\bibitem{DK2007}
D.~Dimarogonas and K.~Kyriakopoulos.
\newblock On the rendezvous problem for multiple nonholonomic agents.
\newblock {\em IEEE Transactions on Automatic Control}, 52(5):916 -- 922, 2007.

\bibitem{DGEH}
D.~V. Dimarogonas, T.~Gustavi, M.~Egerstedt, and X.~Hui.
\newblock On the number of leaders needed to ensure connectivity.
\newblock In {\em Proceedings of the 47th IEEE Conference on Decision and
  Control}, 2008.

\bibitem{Dovrat-TR}
D.~Dovrat and A.~Bruckstein.
\newblock Gathering and collective movement of unicycle a(ge)nts with crude
  sensing capabilities.
\newblock Technical report, Technion, CS, Center for Intelligent Systems TR,
  CIS-2017-02, 2017.

\bibitem{GDEH09}
T.~Gustavi, D.~V. Dimarogonas, M.~Egerstedt, and X.~Hui.
\newblock On the number of leaders needed to ensure connectivity in arbitrary
  dimensions.
\newblock In {\em 17th Mediteranean Conference on Control and Automation},
  2009.

\bibitem{HLG2006}
J.~Han, M.~Li, and L.~Guo.
\newblock Soft control on collective behavior of a group of autonomous agents
  by a shill agent.
\newblock {\em Journal of Systems Science and Compexity}, 19:54--62, 2006.

\bibitem{HW2013}
J.~Han and L.~Wang.
\newblock Nondestructive intervention to multiagent systems through an
  inteligent agent.
\newblock {\em PLoS ONE}, 2013.

\bibitem{JLM}
A.~Jadbabaie, J.~Lin, and A.~S. Morse.
\newblock Coordination of groups of mobile autonomous agents using nearest
  neighbor rules.
\newblock {\em IEEE Transactions on automatic Control}, 48:988--1001, 2002.

\bibitem{TK}
T.~Kailath.
\newblock {\em Linear Systems}.
\newblock Prentice Hall, 1980.

\bibitem{KN}
M.~Klamkin and D.~Newman.
\newblock Cyclic pursuits or the three bugs problem.
\newblock {\em Amer. Math. Monthly}, 78(5):631--639, 1971.

\bibitem{LBF2003}
Z.~Lin, M.~Broucke, and B.~Francis.
\newblock Local control strategies for groups of mobile autonomous agents.
\newblock In {\em Proc. 42nd IEEE Conf. Decision and Control}, pages 1006 --
  1011, 2003.

\bibitem{MarshallPhD}
J.~Marshall.
\newblock {\em COORDINATED AUTONOMY: PURSUIT FORMATIONS OF MULTIVEHICLE
  SYSTEMS}.
\newblock PhD thesis, University of Toronto., 2005.

\bibitem{MBF2004}
J.~Marshall, M.~Broucke, and B.~Francis.
\newblock Formations of vehicles in cylic pursuit.
\newblock {\em IEEE Transactions in Automatic Control}, 49(11), 2004.

\bibitem{MBF2004b}
J.~Marshall, M.~Broucke, and B.~Francis.
\newblock Unicycles in cylic pursuit.
\newblock In {\em Proceedings of the 2004 American Control Conference}, pages
  5344 -- 5349, 2004.

\bibitem{M-L}
K.~Moore and D.~Lucarelli.
\newblock Forced and constrained consensus among cooperating agents.
\newblock In {\em IEEE Conf. on Neworking, Sensing and Control}, pages
  449--454, 2005.

\bibitem{OS-M2003}
R.~Olfati-Saber, A.~Fax, and R.~Murray.
\newblock Agreement problems in networks with directed graphs and switching
  topology.
\newblock In {\em Proceedings of the 42nd IEEE Conference on Decision and
  Control}, 2003.

\bibitem{OS-M-TR2003}
R.~Olfati-Saber, A.~Fax, and R.~Murray.
\newblock Agreement problems in networks with directed graphs and switching
  topology.
\newblock Technical report, California Institute of Technology, Technical
  Report CIT-CDS 03–005, 2003.

\bibitem{Saber}
R.~Olfati-Saber, A.~Fax, and R.~Murray.
\newblock Consensus protocols for networks of dynamic agents.
\newblock In {\em Proceedings of the 2003 Americal Control Conference}, 2003.

\bibitem{Ren2007}
W.~Ren.
\newblock Multi-vehicle consensus with a time-varying reference state.
\newblock {\em Systems and Control Letters}, 56(3):474–483, 2007.

\bibitem{RBMcL}
W.~Ren, R.~Beard, and T.~McLain.
\newblock Coordination variables and consensus building in multiple vehicle
  systems.
\newblock In {\em Proceedings of the Block Island Workshop on Cooperative
  Control, Springer-Verlag}, 2003.

\bibitem{Richardson}
T.~Richardson.
\newblock Non-mutual captures in cyclic pursuit.
\newblock {\em Annals of Mathematics and Artificial Intelligence},
  31:127–146, 2001.

\bibitem{S-B-arxiv}
I.~Segall and A.~Bruckstein.
\newblock Stochastic broadcast control of multi-agent swarms.
\newblock arXiv:1607.04881, 2016.

\bibitem{S-G-TR}
A.~Sinha and D.~Ghose.
\newblock Generalization of linear cyclic pursuit with application to
  rendezvous of multiple autonomous agents.
\newblock Technical report, Indian Institute of Science Bangalore560012, 2005.

\bibitem{VC95}
T.~Vicsek, A.~Czirok, E.~B. Jacob, I.~Cohen, and O.~Schochet.
\newblock Novel type of phase transitions in a system with self-driven
  particles.
\newblock {\em Physical Review Letters}, 75:1226--1229, 1995.

\end{thebibliography}

\end{document}